\documentclass{article}

\usepackage{microtype}
\usepackage{graphicx}
\usepackage{subcaption}
\usepackage{booktabs} % for professional tables

\usepackage{hyperref}
\usepackage{url}
\usepackage{comment}
\usepackage{algorithm}
\usepackage{algpseudocode}
\makeatletter\@namedef{ver@algorithmic.sty}{2009/08/24}\makeatother

\usepackage{tikz}
\usepackage{pgfplots}
\usepackage{float}  % Add this to your preamble

\usepackage[textsize=tiny]{todonotes}
\usepackage{amsthm}
\usepackage{placeins}

\usepackage{adjustbox} 
\usepackage{booktabs,multirow,makecell}
\usepackage[table]{xcolor}  % enables \cellcolor in tables

\definecolor{groupone}{RGB}{255,183,77}  % orange-ish
\definecolor{grouptwo}{RGB}{0,121,107}   % teal-ish
\newcommand{\bestgroupone}[1]{\cellcolor{groupone!20}#1}   % 1st
\newcommand{\bestgrouptwo}[1]{\cellcolor{grouptwo!25}#1} % 2nd

\pgfplotsset{
  tick style={black!60},
  grid style={gray!20,dashed},
}

\usepackage[accepted]{icml2026}

\usepackage{amsmath}
\usepackage{amssymb}
\usepackage{mathtools}
\usepackage{amsthm}

\usepackage{twemojis}
\definecolor{deepgreen}{HTML}{38612D} 
\definecolor{deepred}{HTML}{7B0000}

\usepackage[capitalize,noabbrev]{cleveref}

\theoremstyle{plain}
\newtheorem{theorem}{Theorem}[section]
\newtheorem{proposition}[theorem]{Proposition}

\theoremstyle{definition}

\theoremstyle{remark}

\usepackage[textsize=tiny]{todonotes}

\icmltitlerunning{PAWS: Preference Learning with Advantage-Weighted Segments}

\begin{document}

\twocolumn[
  \icmltitle{PAWS: Preference Learning with Advantage-Weighted Segments}



  \icmlsetsymbol{equal}{*}

  \begin{icmlauthorlist}
    \icmlauthor{Aleksandar Taranovic}{ALR}
    \icmlauthor{Onur Celik}{ALR}
    \icmlauthor{Niklas Freymuth}{ALR}
    \icmlauthor{Ge Li}{ALR}
    \icmlauthor{Serge Thilges}{ALR}
    \icmlauthor{Huy Le}{ALR,Bosch}
    \icmlauthor{Tai Hoang}{ALR}
    \icmlauthor{Rania Rayyes}{IFL}
    \icmlauthor{Gerhard Neumann}{ALR}

  \end{icmlauthorlist}

  \icmlaffiliation{ALR}{Autonomous Learning Robots, Karlsruhe Institute of Technology, Karlsruhe, Germany}
  \icmlaffiliation{IFL}{Institute for Material Handling and Logistics (IFL), Karlsruhe Institute of Technology, Karlsruhe, Germany}
  \icmlaffiliation{Bosch}{Bosch Center for Artificial Intelligence, Renningen, Germany}

  \icmlcorrespondingauthor{Aleksandar Taranovic}{aleksandar.taranovic@kit.edu}

  \icmlkeywords{Preference Learning, Reinforcement Learning from Human Feedback}

  \vskip 0.3in
]



\printAffiliationsAndNotice{}  % no special notice (required even if empty)
\renewcommand{\thefootnote}{*}

\begin{abstract}
Preference-based reinforcement learning (PbRL) learns policies from human trajectory-level comparisons, avoiding explicit reward design and expert demonstrations. Existing methods typically train utility functions on trajectory or segment-level preferences while relying on per-step utility estimates during policy optimization. This training and inference mismatch induces a distribution shift that severely degrades temporal credit assignment and limits policy learning.
We analyze this issue and propose PAWS\footnote{The project webpage with code: \url{https://ataranovic.github.io/PAWS-webpage/}}, a segment-based preference learning method that performs policy updates directly using segment-level advantage functions. By aligning utility training with policy optimization, PAWS preserves trajectory-level preference information and avoids unreliable per-step learning signals. Experiments on simulated robotic manipulation and locomotion tasks demonstrate that PAWS consistently outperforms existing PbRL approaches, highlighting the importance of distribution-consistent preference learning.

\end{abstract}

\section{Introduction}\label{sec::introduction}

Preference-based reinforcement learning (PbRL) aims to learn policies that align with human preferences by leveraging pairwise comparisons between behaviors rather than explicit reward functions \citep{Wirth2017, Christiano2017DeepRL}. In this setting, a human evaluator is presented with two candidate behaviors and provides a binary label indicating which one is preferred under implicitly defined criteria. 
The resulting feedback is often easier and more reliable to provide than demonstrations or manually specified rewards, especially for complex or subjective tasks. 
This is particularly true when demonstration quality varies widely, e.g., teleoperation scenarios with operators of differing expertise.
\begin{figure}[t]
    \centering
    \includegraphics[width=\linewidth]{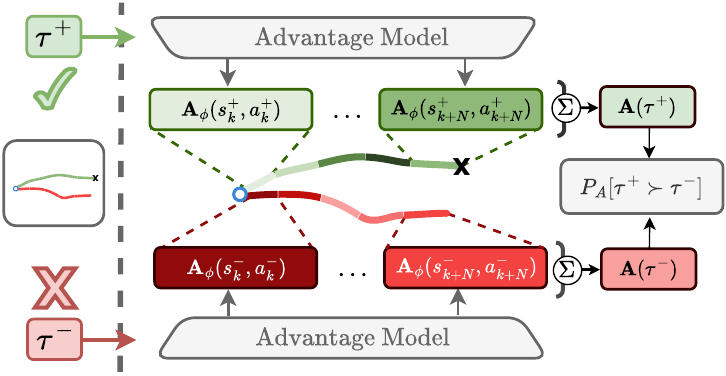}
    \vspace{-2mm}
    \caption{
    \textbf{The Temporal Credit Assignment Problem.}
    Illustration of learning a trajectory-level advantage value $A(\tau)$ from preference data. 
    The advantage model is trained on \textcolor{deepgreen}{preferred} ($\tau^+$) and \textcolor{deepred}{non-preferred} ($\tau^-$) trajectories using the loss $P_{A_\phi}[\tau^+ \succ \tau^-]$.
    This loss depends only on the sum $A_\phi(\tau) = \sum_k A_\phi(s_k, a_k)$ of per-step advantages, which constrains the model at the trajectory level. As a result, many different assignments of individual state-action advantages are consistent with the same preference label~(\cref{fig:trajectories}). During policy optimization, however, the advantage function is queried on individual state-action pairs~\citep{Christiano2017DeepRL,kim2023preference}, inducing a distribution shift between segment-level training and step-level inference.
    }
    \label{fig:adv_diagram}
    \vspace{-3mm}
\end{figure}

In sequential decision-making problems such as robotic control, preferences are typically expressed over trajectory segments, which are sequences of states or state-action pairs, rather than individual steps. Consequently, most PbRL methods train a utility function, such as a reward or advantage model, using segment-level or trajectory-level comparisons.
Policy optimization, however, is usually performed at the level of individual state-action pairs, which requires the learned utility to be inferred on single steps.

The mismatch between segment-level training and step-level inference induces a distribution shift that is largely overlooked in existing PbRL approaches.
\cref{fig:adv_diagram} illustrates how a utility model, in this case the advantage function, is trained. 
Although the utility model is trained to accurately rank complete segments, it is later used to evaluate individual state-action pairs. 
For example, in \cref{fig:trajectories}, we show four possible per-step utility distributions for the preferred and non-preferred segments. Here, a segment $\tau^{+/-}_i$ consists of five steps, with the utility of each step encoded by color intensity, where darker green/red shades indicate higher/lower utility values. The sum of utility of all steps in the segment is the same, but the utility of individual steps can vary significantly.
Thus, many distinct per-step utility assignments can explain the same segment-level preference, leaving temporal credit assignment fundamentally ambiguous. This ambiguity severely limits the ability of the utility model to provide meaningful learning signals during policy optimization, especially when the preferences span a wide range of behavior quality.

\begin{figure}[t]
    \centering
    \includegraphics[width=0.49\linewidth]{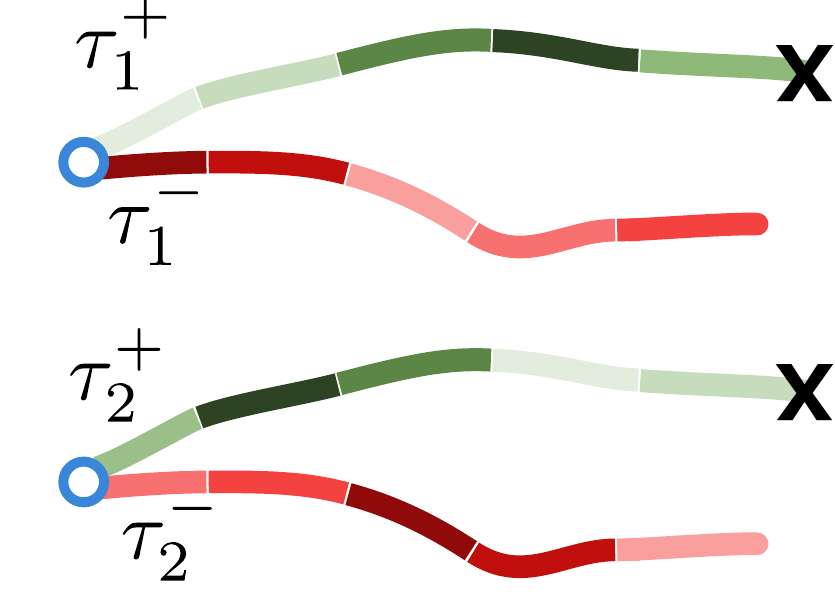}
    \hfill
    \includegraphics[width=0.49\linewidth]{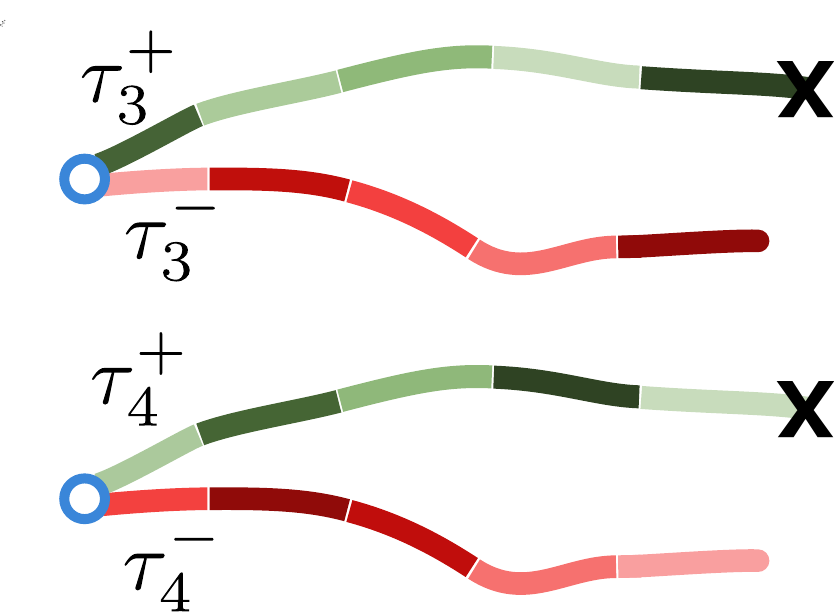}
    \caption{
    \textbf{Ambiguity in Per-Step Credit Assignment.}
    Four pairs of preferred ($\tau^+$) and non-preferred ($\tau^-$) segments are shown, with the intensity of green and red encoding the per-step advantage. The sum of advantages within each segment is identical across the two pairs, so they yield the same trajectory-level preference label, yet the per-step assignments differ markedly. This illustrates the underdetermined nature of temporal credit assignment under the distribution shift between segment-level training and step-level inference, where any consistent per-step assignment is equally compatible with the supervision signal, leaving downstream policy updates exposed to arbitrary choices made by the utility model.
    }
    \label{fig:trajectories}
    \vspace{-4mm}
\end{figure}

This issue is commonly referred to as the temporal credit assignment problem in preference learning \citep{Wirth2017}. Crucially, we argue that this problem does not merely arise from the absence of fine-grained labels, but from a distributional inconsistency between how utilities are trained and how they are used. As we empirically demonstrate, when the same learned utility is queried with segment-level inputs that match its training distribution, prediction quality improves substantially, and preference information becomes significantly more informative for temporal credit assignment.

To address this mismatch, we introduce Preference Learning with Advantage-Weighted Segments (PAWS).
PAWS performs policy optimization directly at the segment level by using a learned advantage function that is trained and queried consistently on trajectory segments. By combining segment-wise advantage estimation with a trust-region-constrained update, PAWS preserves trajectory-level preference information while avoiding unreliable single-step credit assignment. This enables stable and effective policy updates that remain close to the data distribution implied by the preferences. This formulation further enables a principled, data-driven control of policy update strength through the effective sample size of advantage-weighted segments, reducing reliance on manually tuned optimization hyperparameters.

Our contributions are fourfold.
First, we analyze temporal credit assignment in PbRL through the lens of training and inference distribution shift, identifying it as a core limitation of existing methods.
Second, we propose PAWS, a segment-based preference learning approach that aligns utility training with policy optimization, thereby enabling more reliable propagation of preference signals.
Third, we introduce an intuitive, data-driven strategy for setting policy optimization hyperparameters based on the effective sample size of preference-weighted data. 
Finally, we validate our method on a diverse set of simulated robotic manipulation and locomotion tasks, demonstrating consistent performance improvements over established baselines.

\section{Related Work}
\textbf{Preference-based Reinforcement Learning (PbRL).} PbRL leverages comparative feedback as the primary supervision signal~\citep{Wirth2017}. 
Some approaches integrate entropy-based regularization~\citep{pebble} or behavior cloning pretraining, where models are initialized using demonstrations of desired behavior. 
In LLM fine-tuning, a KL-divergence term is frequently introduced to anchor the trained policy to its pretrained counterpart~\citep{ziegler2019finetuning,stiennon2020learning,instruct_gpt,deepseek_r1_2025}. Online preference learning methods, such as those proposed by \citet{verma2024hindsight}, tackle the issue of credit assignment. However, they rely on online interaction and require substantially more data. \citet{gao2024hindsight} propose a similar offline preference learning approach, but it depends on massive amounts of unlabeled data.

\textbf{Direct preference learning.} Recent methods such as 
DPO~\citep{rafailov2023direct}, CPL~\citep{hejna2024contrastive}, IPL~\citep{hejna2023inverse}, DPPO~\citep{an2023direct}, and PPL~\citep{cho2025ppl} implicitly avoid the reward assignment problem.
Instead of learning a reward function, they explicitly optimize the policy likelihood from preferences. However, these methods do not utilize reward models to capture complex dependencies between different state-action pairs. 
As such, they struggle on tasks that require complex reasoning~\citep{ivison2024unpacking, xu2024dpo} and in settings of limited preference data.

\textbf{Feedback Collection.} The process of acquiring preference data varies significantly across methods. Certain frameworks gather feedback online, collecting trajectories or comparisons dynamically during training~\citep{pebble,Christiano2017DeepRL}, while others rely on static, pre-collected datasets~\citep{hejna2024contrastive}.
Preference query generation also varies between methods~\citep{lee2021bpref}. Our work focuses on offline datasets, though PAWS is agnostic to the data collection process and can also be applied to online settings.
A critical consideration in feedback collection is the quality of preference labels~\citep{lee2021bpref}. Many approaches address noisy or conflicting data by aggregating signals from multiple policies or annotators. We focus on informative preferences, where an oracle with access to expert metrics, such as ground-truth rewards or Q-functions, provides labels. Combining preferences with other feedback modalities, such as demonstrations, is an open challenge that has been tackled in several recent works~\citep{biyik2021learning,taranovic2023adversarial,ibarz2018reward}. Similarly, several recent works consider alternative feedback types~\citep{Abdolmaleki2025Learning, myers2021learning}.

\textbf{Reward Function Learning.} The goal in PbRL is to train a reward \citep{ng2000algorithms} or a utility-based value function model~\citep{knox2024models} that aligns with the collected preferences.
The Bradley–Terry (BT) model~\citep{BradleyTerry1952} serves as a foundational approach for many methods. Here, human preference is modeled as a Boltzmann rational distribution~\citep{baker2009action}
over the sum of the discounted rewards of individual time steps. 
However, recent research suggests that human preferences are better captured through regret-based formulations~\citep{knox2024models} where the advantage of the optimal policy is used to model the quality of a segment instead of the rewards.
Most BT-inspired methods optimize log-likelihood objectives, or use logistic separation and hinge loss variants~\citep{taranovic2023adversarial}. 
To better reflect the complexity of human preference structures, reward models are often modeled with a Transformer-based architecture~\citep{nakano2021webgpt,gao2023scaling, zhao2024prefmmt} capable of learning non-Markovian rewards~\citep{kim2023preference}.
When feedback is collected online, the reward model must be continually updated to incorporate new data~\citep{taranovic2023adversarial,pebble}.

\textbf{Policy Optimization.} The choice of the RL method for policy optimization often depends on the data collection paradigm.
For online settings, trust-region-constrained on-policy approaches~\citep{schulman2015trust,schulman2017proximal,abdolmaleki2018maximum, li2024open,otto_iclr2021,hoang2025geometryaware} are commonly used, while off-policy algorithms~\citep{haarnoja2018soft} are typically combined with relabeled rewards~\citep{pebble}. In offline scenarios, algorithms like IQL~\citep{kostrikov2021offline} are frequently adopted to handle the constraints of static datasets.
In the reward-weighted regression setup~\citep{peters2007reinforcement, neumann2008fitted} and related Expectation Maximization (EM)-like approaches~\citep{abdolmaleki2018maximum}, the constrained policy optimization usually leads to a weighted maximum likelihood optimization that avoids querying the credit assignment out of distribution~\citep{kostrikov2021offline, nair2020awac}. 
Our method diverges from conventional PbRL methods by directly learning the advantage function, removing the need for separate reward training. 
Additionally, we propose updating the policy on segments instead of single state-action pairs, similar to some policy optimization methods in episodic RL~\citep{kober2008policy, abdolmaleki2015model, daniel2016hierarchical, li2024open, li2025top}. 
We base the policy update scheme on the EM-like, trust-region-constrained optimization problem~\citep{peters2010relative}, but adapt it to the offline setting.

\section{Method}
\label{sec:method}

We present PAWS, a preference learning method designed to address the distribution shift between utility training and policy optimization in preference-based reinforcement learning. Given previously collected preference data over trajectory segments, we first train an advantage function $A_\phi$. We then optimize the policy by performing updates directly on trajectory segments, using advantage-weighted segment returns rather than per-step utility estimates.

\textbf{Notation.}
We consider a policy $\pi_\theta: \mathcal{S} \times \mathcal{A} \rightarrow \mathbb{R}^+$ parameterized by $\theta$, where $\mathcal{S}$ and $\mathcal{A}$ denote the state and action spaces. The policy induces a distribution over actions conditioned on states. Preference learning is performed by training an advantage function $A_\phi: \mathcal{S} \times \mathcal{A} \rightarrow \mathbb{R}$ with parameters $\phi$, which assigns a scalar utility to state-action pairs.
We assume access to trajectories $\tau_T^i = (s_0^i, a_0^i, \ldots, s_{T-1}^i, a_{T-1}^i)$ of length $T$. Since human preferences are hard to elicit reliably over long trajectories, we operate on fixed-length trajectory segments, which are short enough for consistent comparison while still capturing temporally extended behavior. A segment of length $N$ starting at time step $k$ is denoted as
$\tau_{k:N}^i = (s_k^i, a_k^i, \ldots, s_{k+N-1}^i, a_{k+N-1}^i)$.
Segments are sampled from an unknown data distribution $p_D(\tau)$, from which we assume access to $K$ segment samples. When temporal indexing is not essential, we denote a generic segment simply as $\tau_i$.
Preference data consist of pairs of trajectory segments, where $\tau_i^+$ is preferred over $\tau_i^-$. The preference dataset is denoted as
$
D_{\mathrm{pref}} = \{(\tau_i^+, \tau_i^-)\}_{i=1}^n,
$
and we use $\succ$ to denote the preference relation.

\subsection{Advantage Learning}
\label{sec::adv_learning}

Following \citet{knox2024models}, we model human preferences using an advantage function rather than partial returns, which better reflects how humans evaluate temporally extended behavior. The advantage function is trained on segment-level comparisons, and the likelihood of preferring segment $\tau^+$ over $\tau^-$ is defined as
\begin{align}
P_{A_\phi}[\tau^+ \succ \tau^-] =
\frac{\exp\left( A_\phi(\tau^+)\right)}
{\exp\left(A_\phi(\tau^+)\right)
+ \exp\left(A_\phi(\tau^-)\right)},
\end{align}
where
$A_\phi(\tau) = \sum_{t} A_\phi(s_t, a_t)$.
The advantage values of both segments are aggregated, and training reduces to binary cross-entropy optimization over the difference in cumulative advantage predictions. The resulting loss is
\begin{equation}\label{eq::AdvOpt}
\mathcal{L}_{\mathrm{pref}}(\phi)
= -\frac{1}{n} \sum_{i=1}^n
\log \sigma\left(A_\phi(\tau_i^+) - A_\phi(\tau_i^-)\right).
\end{equation}
Importantly, this objective constrains the advantage function only at the segment level. While $A_\phi$ is trained to correctly rank trajectory segments, many different per-step advantage assignments can yield identical segment-level sums, as shown in \cref{fig:trajectories}. 
Thus, the learned advantage function is underconstrained at the level of individual state-action pairs.
Most existing PbRL methods subsequently query this advantage or reward model at the per-step level during policy optimization, thereby inducing a distribution shift between the segment-level training and the step-level inference. This mismatch leads to ambiguous and unreliable temporal credit assignment. In contrast, PAWS preserves distributional consistency by using the learned advantage function directly on trajectory segments during policy optimization.

We consider two architectures for $A_\phi$: an encoder-only Transformer \citep{vaswani2017attention} and a simple multilayer perceptron (MLP). For the Transformer, each state-action pair in a segment is encoded and passed through an MLP head to produce per-step advantage values $A_\phi(s_k, a_k), \ldots, A_\phi(s_{k+N-1}, a_{k+N-1})$, which are then aggregated via summation to obtain $A_\phi(\tau)$. The MLP directly maps individual state-action pairs to scalar advantage values.
To prevent overfitting, we employ early stopping based on validation accuracy or terminate training once a predefined maximum number of update steps is reached. \cref{sec::evaluations} compares the performance of both architectures.

\subsection{Policy Update}

We use the advantage function $A_\phi(\tau)$ to first extract an optimal 
segment distribution $p^*(\tau)$ such that segments with higher advantage have higher likelihood while staying close to the data distribution. Subsequently, we infer the policy $\pi(a|s)$ from $p^*(\tau)$ using a maximum likelihood fitting step. 

\textbf{Obtaining the segment distribution.}
We optimize the segment distribution using a constrained optimization problem in the segment space $\tau$, 
\begin{equation}
\label{eq::PolicyOpt}
\begin{aligned}
&\max_p \int p(\tau) A_\phi(\tau) d\tau \\
&\text{s.t.}~~ \text{KL}\left(p(\tau) \,||\, p_D(\tau)\right) \leq \epsilon 
\text{ and } \int p(\tau) d\tau = 1.
\end{aligned}
\end{equation}
This optimization problem has two properties that are crucial in our setting. 
First, it includes a trust-region that constrains the new segment distribution to be close to the data distribution $p_D$, and second, finding a new segment distribution is a pure offline process and does not involve generating new segments using the policy. 
Both benefits are important in the offline policy optimization step to avoid out-of-distribution samples~\citep{kostrikov2021offline, nair2020awac}.
Here, we leverage this formalism to extract a reweighted segment distribution using preference data only.
The derivations in this section adapt the relative entropy policy search framework of \citet{peters2010relative} to segment-wise updates in the offline preference setting.
Using Lagrangian optimization \citep{boyd2004convex}, we can obtain the optimal solution to the optimization problem in~\cref{eq::PolicyOpt} that we summarize in the following proposition.
\begin{proposition} [Optimal Segment Distribution]
\label{prop::opt_sol}
    \citep{peters2010relative} Given an advantage function \mbox{$A_\phi(\tau)=\sum_t A_\phi(s_t,a_t)$}, samples from a preference distribution $\tau\sim p_D(\tau)$ and a fixed Lagrange multiplier $\lambda$, the optimal policy to the optimization problem in~\cref{eq::PolicyOpt} is 
    \begin{align}\label{eq::opt_sol}
        p^*(\tau) \propto p_D(\tau)\exp{\left(\frac{1}{\lambda}A_\phi(\tau)\right)}.
    \end{align}
\end{proposition}
Proposition \ref{prop::opt_sol} shows that the new segment distribution $p^*$ is proportional to the preference distribution $p_D$, but reweighted with the exponential of the advantage estimate of the segment. 
Intuitively, it increases the likelihood of the data distribution where the advantage is high.
In \cref{appdx::derivation_opt_distr}, we provide the detailed derivation for the segment-based case as considered here.

\textbf{Extracting the policy from the segment distribution.}
We can obtain a policy $\pi_\theta$ by fitting the segment distribution 
\begin{align}
    p_{\theta}(\tau) = p_D(s_0) \prod_{t=0}^{N-1} p(s_{t+1}|s_t, a_t) \pi_{\theta}(a_t|s_t) \label{eq::segment_prob_def}
\end{align} 
induced by policy $\pi_\theta$ to the desired segment distribution $p^*(\tau)$. This can be achieved by minimizing the KL-projection 
$$\theta^* = \operatorname{argmin}_{ \theta} \quad \textrm{KL}(p^*(\tau) \,\|\, p_{\theta}(\tau)).$$
After simplifying and re-arranging terms (see \cref{appdx::derivation_weighted_mle}), we can realize that minimizing the projection is equivalent to maximizing the following objective 
\begin{align}\label{eq::ml_loss}
    \mathcal{L}(\theta) = \mathbb{E}_{p_D}\left[w(\tau)\log p_\theta(\tau)\right], 
\end{align}
where the importance weights 
\begin{equation}\label{eq::ImpWeightDef}
    w(\tau) = \frac{p^*(\tau)}{p_D(\tau)}\propto\exp{\left(\frac{1}{\lambda}A_\phi(\tau)\right)}
\end{equation} are given by the exponentially transformed segment advantages.
We can further simplify this optimization by using~\cref{eq::segment_prob_def} to factorize $p_\theta(\tau)$ and, as the $\log$ turns the product into sums, removing all terms that do not depend on $\theta$. Hence, we arrive at the following weighted maximum likelihood problem
\begin{equation}\label{eq::grad_ml_loss_sa}
        \begin{split}
        \mathcal{L}(\theta) = \sum_{\tau \in D} \sum_{(s_t,a_t)\in \tau} & \exp\bigg(\frac{A_\phi(\tau)}{\lambda}\bigg) 
         \log\pi_\theta(a_t|s_t).
        \end{split}
    \end{equation}
Similar optimization schemes for obtaining a parametric policy were suggested in online and offline RL algorithms, for example, by \citet{nair2020awac, abdolmaleki2018maximum}. However, these schemes typically use weightings that are computed from single state-action pairs instead of advantages of whole segments. 

\subsection{Computing the Lagrange Multiplier \texorpdfstring{$\lambda$}{lambda}} \label{sec::opt_lambda}
Manually setting the Lagrange multiplier $\lambda$ can be difficult, because for too high $\lambda$ values, the importance weights in~\cref{eq::ImpWeightDef} become uniform, forcing the optimization to consider all data points, including undesired bad examples.
Too small $\lambda$ values exploit data points with high advantage, such that the effectively used number of data points collapses to only a few.

The Lagrange multiplier $\lambda$ directly depends on the parameter $\epsilon$ that upper bounds how much the policy $\pi$ is allowed to deviate from the preference distribution $p_D$.
Finding a good value for $\lambda$ is therefore crucial to balance the exploitation of data points with high advantage values, especially in scenarios, where data is scarce and the advantage function $A_\phi(\tau)$ might not provide accurate estimates.
We follow prior works \citep{peters2010relative, daniel2016hierarchical} and optimize for $\lambda$ by minimizing the dual function as stated in the next proposition. 
\begin{proposition}[Optimal Lagrange multiplier] \citep{peters2010relative}
    Minimizing the dual function 
    \begin{align}\label{eq::dual}
        g(\lambda) = \lambda\epsilon +\lambda\log \int p_D(\tau)\exp{\left(\frac{A_\phi(\tau)}{\lambda}\right)}d\tau 
    \end{align}
    yields the optimal Lagrange multiplier $\lambda^*$.
\end{proposition}
\cref{appdx::derivation_dual_fct} provides a detailed derivation.
The expectation over $p_D(\tau)$ in the dual function can be easily approximated via a Monte Carlo estimate using the available samples.

The dual function optimization directly connects the choice of $\lambda$ with the KL-bound $\epsilon$.
However, tuning $\epsilon$ can often be problem-specific as the KL depends on the problem's action dimensionality. 
Moreover, good KL values typically depend on the dataset size, where larger datasets typically allow for larger step sizes and thus KL values than small datasets. 
Hence, instead of tuning the KL bound $\epsilon$ by hand, we propose to automatically adapt $\epsilon$ based on the desired number of effective samples
\begin{align}\label{eq::n_eff}
    n_{\mathrm{eff}} = \frac{\left(\sum_iw_i\right)^2}{\sum_iw_i^2},
    w_i= \exp\left(\frac{1}{\lambda}\sum_t  A_\phi(s^i_t,a^i_t)\right),
\end{align}
which is measured based on the importance weights $w_i$ and therefore depends on $\lambda$.
For a desired value $n^*_{\mathrm{eff}}$ we can automatically find the corresponding $\epsilon$ through an iterative process, as shown in \cref{alg:offline-paws}.
Choosing a value for $n^*_{\mathrm{eff}}$ is often more intuitive than defining $\epsilon$, as the number of effective samples corresponds to the number of samples that are effectively used for updating the policy parameters $\theta$. 
For example, a value of $n^*_{\mathrm{eff}}=10\%$ for a dataset with $500$ preferences would mean that segments from $50$ preferences meaningfully contribute to each policy update.
This interpretation is more straightforward than defining an upper bound on the trust region $\epsilon$ of the KL constraint, whose value depends on the dimensionality of the action space. 

\section{Evaluations}\label{sec::evaluations}

\begin{table*}[t]
\centering
\small
\setlength\tabcolsep{2pt}
\caption{\textbf{Meta-World:} Task success (\%) $\pm$ 2SE. Best results for $n=50$ and $n=500$ are highlighted with \textbf{\textcolor{groupone}{orange}} and \textbf{\textcolor{grouptwo}{teal}}, respectively. The last two rows report the average performance across tasks and the relative improvement (\%) over Behavior Cloning (BC).}
\label{tab:sr_metaworld_transposed}
\begin{adjustbox}{max width=\textwidth}
\footnotesize
\renewcommand{\arraystretch}{1.4}
\begin{tabular}{lccccccccccccccccc}
\toprule
\textbf{Task} & \multicolumn{2}{c}{\textbf{BC}} & \multicolumn{2}{c}{\textbf{P-IQL}} & \multicolumn{2}{c}{\textbf{CPL}} & \multicolumn{2}{c}{\textbf{CPL+KL}} & \multicolumn{2}{c}{\textbf{Pref Trans.}} & \multicolumn{2}{c}{\textbf{IPL}} & \multicolumn{2}{c}{\textbf{PAWS (Trans.)}} & \multicolumn{2}{c}{\textbf{PAWS (MLP)}} \\
\cmidrule(lr){2-3}\cmidrule(lr){4-5}\cmidrule(lr){6-7}\cmidrule(lr){8-9}\cmidrule(lr){10-11}\cmidrule(lr){12-13}\cmidrule(lr){14-15}\cmidrule(lr){16-17}
\textit{\#Preferences} & \textit{50} & \textit{500} & \textit{50} & \textit{500} & \textit{50} & \textit{500} & \textit{50} & \textit{500} & \textit{50} & \textit{500} & \textit{50} & \textit{500} & \textit{50} & \textit{500} & \textit{50} & \textit{500} \\
\midrule
Button Press & 69$\pm$3 & 67$\pm$3 & 72$\pm$6 & 77$\pm$5 & 70$\pm$8 & \bestgrouptwo{87$\pm$5} & 67$\pm$7 & 86$\pm$3 & 71$\pm$5 & 77$\pm$2 & 66$\pm$8 & 69$\pm$4 & 80$\pm$5 & 84$\pm$4 & \bestgroupone{82$\pm$6} & 82$\pm$4 \\
Door Open & 48$\pm$6 & 52$\pm$3 & 52$\pm$8 & 83$\pm$3 & 36$\pm$9 & 71$\pm$9 & 44$\pm$11 & 77$\pm$6 & 62$\pm$6 & 87$\pm$2 & 47$\pm$8 & 59$\pm$4 & \bestgroupone{70$\pm$8} & 96$\pm$1 & 65$\pm$15 & \bestgrouptwo{98$\pm$1} \\
Drawer Open & 54$\pm$5 & 60$\pm$6 & \bestgroupone{58$\pm$6} & 71$\pm$2 & 53$\pm$5 & \bestgrouptwo{79$\pm$3} & 56$\pm$7 & 78$\pm$5 & 45$\pm$3 & 71$\pm$1 & 44$\pm$5 & 55$\pm$3 & 44$\pm$11 & 74$\pm$3 & 40$\pm$12 & 75$\pm$3 \\
Faucet Close & 51$\pm$6 & 66$\pm$3 & 59$\pm$7 & 80$\pm$3 & 53$\pm$7 & 64$\pm$4 & 48$\pm$6 & 63$\pm$3 & 59$\pm$3 & 85$\pm$2 & 54$\pm$6 & 67$\pm$4 & 67$\pm$10 & \bestgrouptwo{87$\pm$3} & \bestgroupone{68$\pm$8} & \bestgrouptwo{87$\pm$3} \\
Lever Pull & \bestgroupone{36$\pm$7} & 44$\pm$3 & 31$\pm$5 & 38$\pm$4 & 28$\pm$8 & 46$\pm$3 & 29$\pm$8 & 47$\pm$4 & 34$\pm$3 & 47$\pm$2 & 35$\pm$4 & 54$\pm$3 & 28$\pm$5 & \bestgrouptwo{58$\pm$5} & 30$\pm$12 & 55$\pm$4 \\
Peg Insert Side & \bestgroupone{33$\pm$7} & 48$\pm$3 & 31$\pm$6 & 78$\pm$5 & 32$\pm$6 & 68$\pm$4 & 27$\pm$8 & 67$\pm$4 & \bestgroupone{33$\pm$3} & 80$\pm$1 & 29$\pm$5 & 49$\pm$3 & 24$\pm$7 & 81$\pm$4 & 23$\pm$9 & \bestgrouptwo{82$\pm$3} \\
Plate Slide & 47$\pm$5 & 50$\pm$6 & 50$\pm$4 & 74$\pm$7 & 42$\pm$5 & 65$\pm$3 & 42$\pm$7 & 65$\pm$5 & \bestgroupone{55$\pm$3} & 78$\pm$2 & 54$\pm$8 & 56$\pm$3 & 48$\pm$6 & 74$\pm$5 & 49$\pm$9 & \bestgrouptwo{78$\pm$5} \\
Push Back & 25$\pm$4 & 37$\pm$2 & 23$\pm$5 & 43$\pm$3 & 25$\pm$7 & 45$\pm$2 & 18$\pm$5 & 43$\pm$4 & 25$\pm$2 & 48$\pm$2 & \bestgroupone{26$\pm$4} & 38$\pm$2 & \bestgroupone{26$\pm$5} & \bestgrouptwo{56$\pm$3} & 24$\pm$4 & 53$\pm$3 \\
Sweep Into & 30$\pm$6 & 58$\pm$3 & 35$\pm$5 & 66$\pm$5 & 26$\pm$9 & 66$\pm$5 & 31$\pm$6 & 64$\pm$6 & 31$\pm$2 & 66$\pm$2 & 31$\pm$5 & 57$\pm$2 & 35$\pm$8 & \bestgrouptwo{74$\pm$3} & \bestgroupone{36$\pm$9} & \bestgrouptwo{74$\pm$4} \\
Window Close & 69$\pm$8 & 91$\pm$2 & 78$\pm$9 & 97$\pm$1 & 64$\pm$7 & 85$\pm$5 & 59$\pm$8 & 83$\pm$4 & 86$\pm$0 & 96$\pm$1 & 71$\pm$6 & 94$\pm$2 & \bestgroupone{94$\pm$5} & 98$\pm$1 & 91$\pm$5 & \bestgrouptwo{99$\pm$0} \\
\midrule
Avg. over tasks     & 46.2      & 57.3      & 48.9      & 70.7      & 42.9      & 67.6      & 42.1      & 67.3      & 50.1      & 73.5      & 45.7      & 59.8      & \bestgroupone{51.6}      & {78.2}   & 50.8        & \bestgrouptwo{78.3}        \\
\midrule
Improvement (\%) & 0.0 & 0.0 & 5.8 & 23.4 & -7.1 & 18.0 & -8.9 & 17.5 & 8.4 & 28.3 & -1.1 & 4.4 & \bestgroupone{11.7} & {36.5} & 10.0 & \bestgrouptwo{36.6} \\
\bottomrule
\end{tabular}
\end{adjustbox}
\end{table*}

We analyze the learning behavior of PAWS using two different parameterizations for the advantage function, a Transformer-based and MLP-based advantage function, which we denote as PAWS (Trans.) and PAWS (MLP), respectively.
As baselines, we evaluate Behavior Cloning (BC), P-IQL, CPL~\citep{hejna2024contrastive}, CPL+KL, which is related to PPL~\citep{cho2025ppl}, Preference Transformer~\citep{kim2023preference}, and IPL~\citep{hejna2023inverse}. IQL~\citep{kostrikov2021offline} has been used to optimize a utility function in the PbRL setting in several recent works~\citep{hejna2024contrastive,kim2023preference} using slightly different objectives. In our case, we implement it using the advantage as the target~\citep{hejna2024contrastive}. Following previous naming conventions, we call the resulting baseline P-IQL.  We additionally evaluate DPPO~\citep{an2023direct}, which is designed to rely on a large unlabeled trajectory dataset alongside preferences. Without this auxiliary dataset, which is not available in our offline setup, DPPO underperforms all other baselines on every task we evaluated, achieving e.g.\ only $9\%$ success on Sweep Into and $1\%$ on Peg Insert Side with 500 preferences. We therefore relegate the full DPPO results to \cref{appdx:DPPO} and exclude DPPO from the main tables for clarity.
We further compare performance under two preference budgets, namely $50$ and $500$ preferences, sampled from the same larger dataset. 
We follow the sparse sampling setup proposed in CPL~\citep{hejna2024contrastive}. From a pool of $K$ segments, we form each preference pair by drawing one segment from the first half and one from the second half.
Our evaluations are performed on 10 different Meta-World tasks~\citep{yu2020meta} and four locomotion tasks from D4RL~\citep{fu2020d4rl,towers2024gymnasium}.
More details about these tasks are provided in \cref{appdx:PreferenceDataset}.

For each task, we construct a dataset containing an equal number of samples collected from four policies of varying quality. This data generation procedure more realistically reflects differences in human expertise, as commonly observed in settings such as teleoperation or human-in-the-loop control. In contrast, several baseline methods \citep{hejna2024contrastive,cho2025ppl} generate suboptimal behavior by injecting uncorrelated Gaussian noise into expert actions, an assumption that is unlikely to capture structured deviations arising from lower skill levels. To obtain policies of different quality, we train an expert policy using Soft Actor-Critic (SAC) \citep{haarnoja2018soft} and periodically save checkpoints throughout training. From these checkpoints, we automatically select four policies whose performance spans a broad range, from poor to high-quality behavior. A detailed performance evaluation of the resulting policies is provided in \cref{tab:rollout_metaworld} in \cref{appdx:PreferenceDataset}.

Preference labels are generated using the highest-performing policy as a proxy for an expert annotator. We derive preferences from the policy’s log action probabilities. Under SAC, the optimal policy follows a Boltzmann distribution with respect to the advantage function, such that log probabilities correspond to advantage values~\citep{haarnoja2017reinforcement}. This induces a ranking based on relative action quality rather than absolute Q-values, which better reflects how humans compare behaviors. Consequently, log probabilities provide a principled and efficient surrogate for expert preference labels.

\textbf{Meta-World Tasks.}
We modified the Meta-World tasks in two ways, as proposed in our baseline~\citep{hejna2024contrastive}. Namely, we randomize the initial hand position and also remove proprioceptive history from the observation. The success rates on Meta-World tasks are reported in \cref{tab:sr_metaworld_transposed}, and the mean returns are provided in \cref{tab:transposed_performance} in \cref{appdx:Average returns}. Our evaluation methodology follows CPL. Specifically, we execute 25 rollouts per evaluation point, smooth the results by averaging over eight adjacent points, and report the peak performance.
All experiments are conducted over 10 random seeds, and we report the mean along with two times the standard error. All methods use the same MLP policy architecture; for PAWS, we fix the effective sample size to \( n_{\mathrm{eff}} = 10\% \) across all evaluations. For the baselines, we use the implementations and recommended hyperparameters provided by \citet{hejna2024contrastive} (see \cref{app:hp}).
Overall, our method achieves higher average success rates and higher mean returns across tasks (\cref{tab:transposed_performance}). In the low-data regime with only 50 preference queries, several baseline methods degrade in performance and, in some cases, perform worse than behavior cloning. In contrast, PAWS outperforms behavior cloning and all other baselines in both low- and high-data settings in the majority of the tasks.

\textbf{Locomotion Tasks.}
We evaluate four different locomotion tasks, namely Ant, HalfCheetah, Hopper, and Walker2d. The average returns from 25 rollouts over 10 seeds are presented in \cref{tab:Mujoco}. The evaluation procedure is the same as for the Meta-World tasks. Moreover, we use the same hyperparameters and policy network size, and we use $n_{\mathrm{eff}}=30\%$ for PAWS for all evaluations. On all locomotion tasks, PAWS achieves higher returns than all of the baselines for both preference budgets.

\begin{table*}[t]
\centering
\small
\setlength\tabcolsep{2pt}
\caption{\textbf{Locomotion Tasks:} Average episode returns over 25 rollouts and 10 random seeds $\pm$ 2SE. Best results for $n{=}50$ and $n{=}500$ preferences are highlighted with \textbf{\textcolor{groupone}{orange}} and \textbf{\textcolor{grouptwo}{teal}}, respectively. Both PAWS variants outperform all baselines across all four locomotion tasks and both preference budgets. The same hyperparameters and policy architecture are used for all methods, with $n^*_{\mathrm{eff}}=30\%$ for PAWS.}
\label{tab:Mujoco}
\begin{adjustbox}{max width=\textwidth}
\footnotesize
\renewcommand{\arraystretch}{1.4}
\begin{tabular}{lccccccccccccccccc}
\toprule
\textbf{Task} & \multicolumn{2}{c}{\textbf{BC}} & \multicolumn{2}{c}{\textbf{P-IQL}} & \multicolumn{2}{c}{\textbf{CPL}} & \multicolumn{2}{c}{\textbf{CPL+KL}} & \multicolumn{2}{c}{\textbf{Pref Trans.}} & \multicolumn{2}{c}{\textbf{IPL}} & \multicolumn{2}{c}{\textbf{PAWS (Trans.)}} & \multicolumn{2}{c}{\textbf{PAWS (MLP)}} \\
\cmidrule(lr){2-3}\cmidrule(lr){4-5}\cmidrule(lr){6-7}\cmidrule(lr){8-9}\cmidrule(lr){10-11}\cmidrule(lr){12-13}\cmidrule(lr){14-15}\cmidrule(lr){16-17}
\textit{\#Preferences} & \textit{50} & \textit{500} & \textit{50} & \textit{500} & \textit{50} & \textit{500} & \textit{50} & \textit{500} & \textit{50} & \textit{500} & \textit{50} & \textit{500} & \textit{50} & \textit{500} & \textit{50} & \textit{500} \\
\midrule
Ant & 546$\pm$43 & 593$\pm$34 & 567$\pm$37 & 847$\pm$12 & 508$\pm$34 & 563$\pm$33 & 518$\pm$22 & 581$\pm$8 & 505$\pm$17 & 784$\pm$17 & 384$\pm$16 & 634$\pm$10 & 685$\pm$22 & 850$\pm$10 & \bestgroupone{707$\pm$26} & \bestgrouptwo{882$\pm$16} \\
HalfCheetah & 1019$\pm$29 & 1085$\pm$66 & 1029$\pm$20 & 1031$\pm$107 & 998$\pm$45 & 950$\pm$87 & 998$\pm$35 & 948$\pm$93 & 945$\pm$25 & 968$\pm$85 & 690$\pm$3 & 1118$\pm$26 & \bestgroupone{1063$\pm$26} & \bestgrouptwo{1483$\pm$92} & 810$\pm$93 & 1450$\pm$251 \\
Hopper & 456$\pm$45 & 401$\pm$22 & 397$\pm$39 & 571$\pm$65 & 391$\pm$39 & 398$\pm$35 & 451$\pm$61 & 405$\pm$32 & 500$\pm$28 & 552$\pm$48 & 436$\pm$85 & 385$\pm$48 & 484$\pm$41 & 563$\pm$60 & \bestgroupone{512$\pm$57} & \bestgrouptwo{637$\pm$50} \\
Walker2d & 167$\pm$33 & 266$\pm$30 & 116$\pm$33 & 947$\pm$32 & 119$\pm$28 & 446$\pm$55 & 123$\pm$31 & 476$\pm$77 & 150$\pm$40 & 726$\pm$56 & 160$\pm$76 & 251$\pm$34 & 208$\pm$82 & 997$\pm$21 & \bestgroupone{235$\pm$70} & \bestgrouptwo{1023$\pm$16} \\
\bottomrule
\end{tabular}
\end{adjustbox}
\vspace{-2.5mm}
\end{table*}

\textbf{Human-Labeled Data.}
Humans may provide preferences that differ from those of an oracle~\citep{pebble}. Therefore, for two Meta-World tasks, Button Press and Door Open, we collected 50 pairwise comparisons per task from each of 10 non-author human labelers.  Further details and statistical analysis, as well as the Graphical User Interface used to collect the preferences, are provided in \cref{app:human}.

PAWS and all the baselines used the same hyperparameters as in the main evaluation (\cref{tab:sr_metaworld_transposed}). Each of the 10 seeds was trained on the 50 comparisons provided by a single distinct labeler, so the reported variance reflects both seed and labeler variability. The mean success rates over 10 seeds are presented in \cref{tab:sr_human_labels_transposed}. Overall, the results are in line with the preferences provided by the oracle and show the benefits of using PAWS, with PAWS (MLP) achieving the highest success rate on both tasks and PAWS (Trans.) ranking second on Door Open.

\begin{table*}[t]
\centering
\small
\setlength\tabcolsep{4pt}
\caption{Task success (\%) $\pm$ 2SE with \textbf{human-collected preferences} from 10 non-author participants on Button Press and Door Open tasks. Each participant labeled 50 pairwise comparisons per task.}
\label{tab:sr_human_labels_transposed}
\begin{adjustbox}{max width=\textwidth}
\footnotesize
\renewcommand{\arraystretch}{1.4}
\begin{tabular}{lccccccccc}
\toprule
\textbf{Task} & \textbf{BC} & \textbf{P-IQL} & \textbf{CPL} & \textbf{CPL+KL} & \textbf{Pref Trans.} & \textbf{IPL} & \textbf{PAWS (Trans.)} & \textbf{PAWS (MLP)} \\
\midrule
Button Press & 44.7$\pm$2.4 & 53.2$\pm$3.7 & 47.5$\pm$7.4 & 49.0$\pm$10.1 & 55.2$\pm$5.6 & 47.8$\pm$4.2 & \bestgroupone{57.2$\pm$5.6} & 56.2$\pm$5.6 \\
Door Open & 73.6$\pm$2.4 & 64.0$\pm$3.7 & 55.6$\pm$9.2 & 54.8$\pm$8.0 & 76.9$\pm$4.4 & 76.2$\pm$3.3 & 82.2$\pm$3.8 & \bestgroupone{86.0$\pm$3.0} \\
\bottomrule
\end{tabular}
\end{adjustbox}
\end{table*}

\textbf{Varying the Number of Effective Samples.}
As discussed in \cref{sec::opt_lambda}, the effective sample size \( n_{\mathrm{eff}} \) controls how far the updated policy is allowed to deviate from the provided data distribution. Larger values of \( n_{\mathrm{eff}} \) enforce higher overlap between the updated policy and the segment distribution, while smaller values permit more selective updates that emphasize only segments with high advantage.

We study the effect of  \( n_{\mathrm{eff}} \) across different data regimes. In \cref{fig:ess_1000}, we consider a higher-data setting with 500 preferences and report aggregated success rates over the \textit{Peg Insert Side}, \textit{Sweep Into}, and \textit{Drawer Open} tasks for six relative values of \( n_{\mathrm{eff}} \), ranging from \(5\%\) to \(50\%\) of the available samples. In this regime, performance improves as the effective sample size decreases, peaking at \( n_{\mathrm{eff}} = 10\% \), indicating that with sufficient data coverage the policy can benefit from more aggressive, advantage-driven updates.

In contrast, \cref{fig:ess_100} shows results for the same range of relative effective sample sizes on the same three tasks using 50 preferences. In this low-data regime, performance is maximized at a substantially larger value of \( n_{\mathrm{eff}}\). In this case, using too few effective samples leads to poor estimation of the policy, whereas larger values of \( n_{\mathrm{eff}} \) keep the policy closer to the provided data distribution and yield more stable improvements.

Taken together, these results suggest that larger datasets permit more aggressive policy updates, whereas smaller datasets benefit from conservative updates that rely on broader data support. The observed trend further indicates that the absolute number of effective samples, rather than their relative proportion, may provide a more robust criterion for controlling policy updates across different data regimes. We consider this a promising direction for future work.
In this study, we fix the effective sample size as a relative proportion of the available data to simplify evaluation and ensure consistent behavior across tasks with different dataset sizes. Despite this simplification, the chosen hyperparameters yield strong and stable performance across all experiments, although further gains for PAWS may be achievable with adaptive or absolute effective sample size selection, as suggested by the analysis in \cref{fig:ess_1000} and \cref{fig:ess_100}. We note that results in these figures are averaged over 5 random seeds and may therefore differ slightly from those reported in \cref{tab:sr_metaworld_transposed}.

\begin{table}[t]
\centering
\scriptsize
\setlength\tabcolsep{6pt}
\caption{Aggregated performance (success rates \%) across all 10 Meta-World tasks for two architectures (MLP, Transformer) and two policy update granularities (segment-level, used by PAWS, and state-action level, denoted \textbf{State}).}
\label{tab:aggregated_performance}
\begin{tabular}{lcccc}
\toprule
\textit{\#Pref.} & \textbf{PAWS(MLP)} & \textbf{PAWS(Trans.)} & \textbf{State(MLP)} & \textbf{State(Trans.)} \\
\midrule
\textit{50} & 51 & \textbf{52} & 40 & 44 \\
\textit{500} & \textbf{78} & \textbf{78} & 67 & 63 \\
\bottomrule
\end{tabular}
\vspace{-1.5mm}
\end{table}

\begin{figure*}[!htbp]
    \centering
    \begin{subfigure}[t]{0.32\linewidth}
        \centering
        \includegraphics[width=\linewidth]{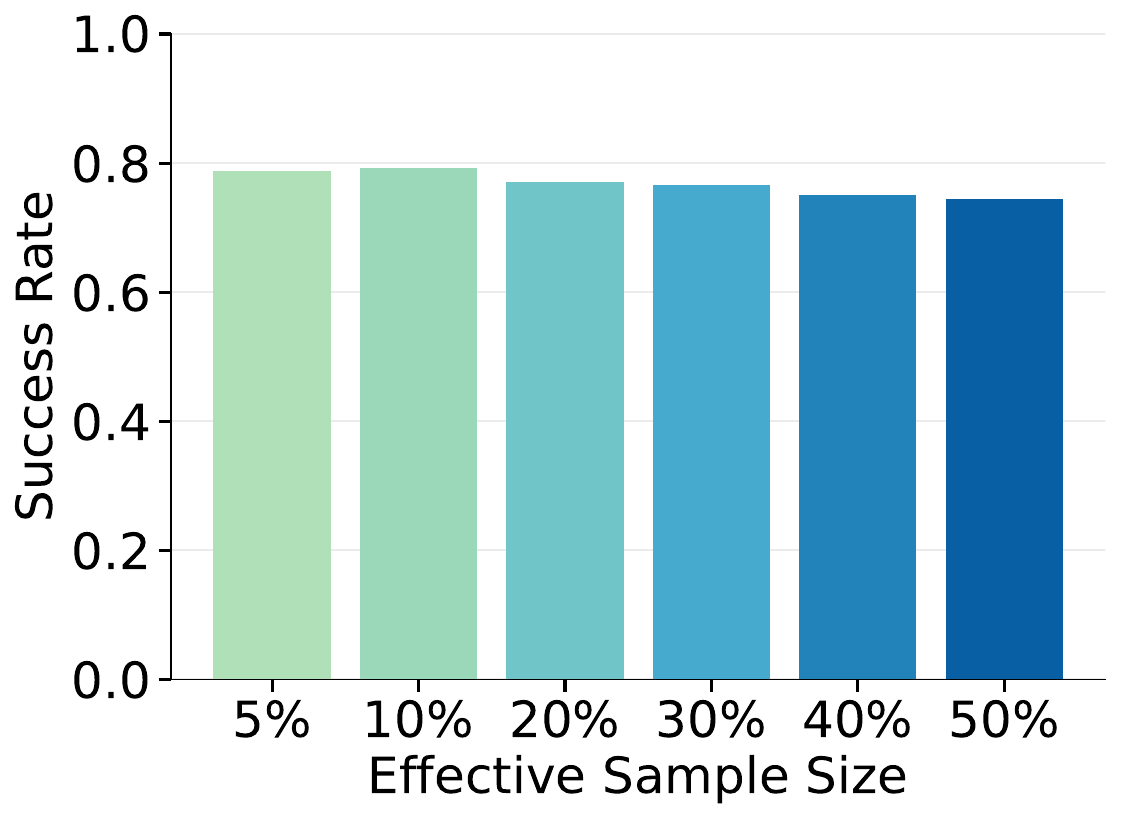}
        \caption{$n=500$ preferences.}
        \label{fig:ess_1000}
    \end{subfigure}
    \hfill
    \begin{subfigure}[t]{0.32\linewidth}
        \centering
        \includegraphics[width=\linewidth]{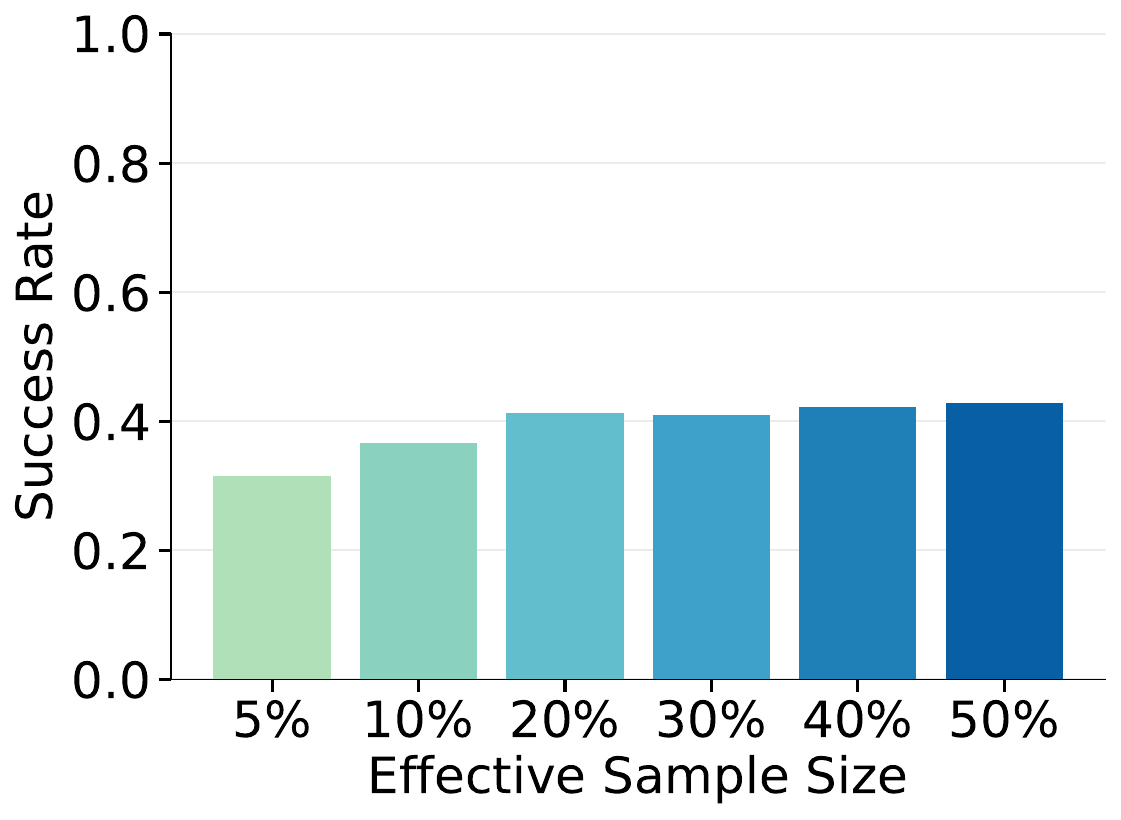}
        \caption{$n=50$ preferences.}
        \label{fig:ess_100}
    \end{subfigure}
    \hfill
    \begin{subfigure}[t]{0.32\linewidth}
        \centering
        \includegraphics[width=\linewidth]{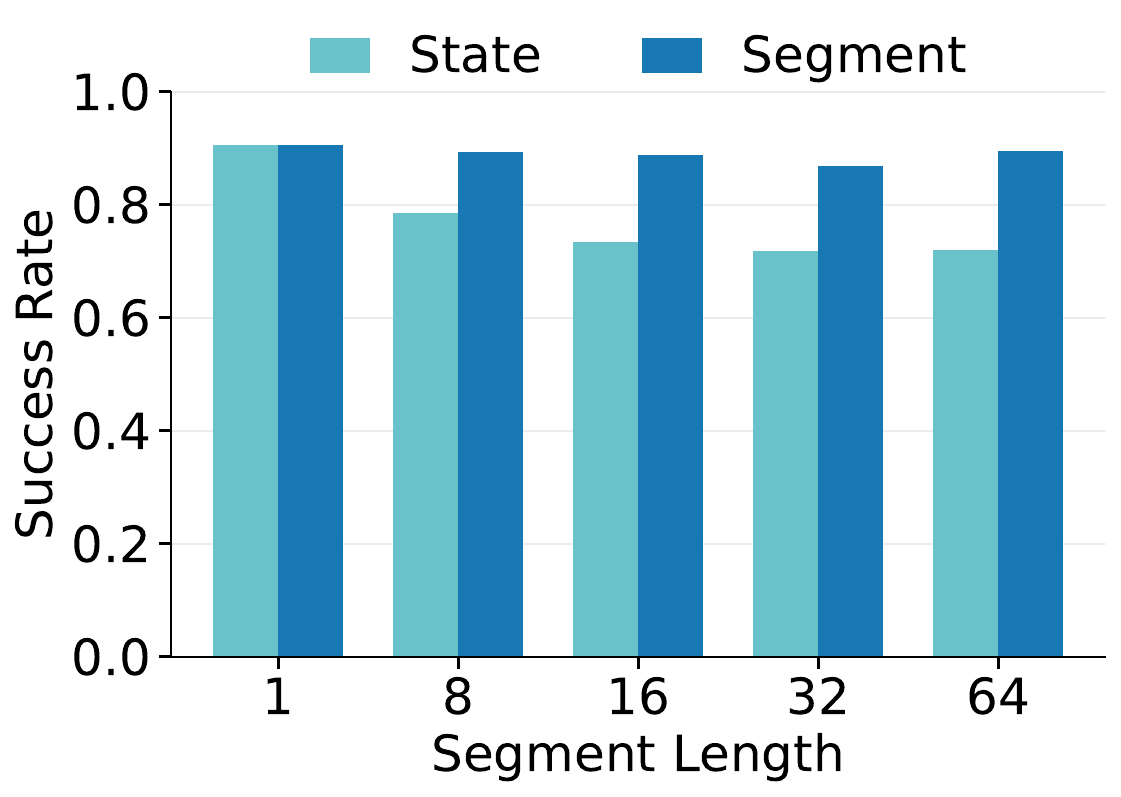}
        \caption{Varying segment length.}
        \label{fig:diff_segment_lengths}
    \end{subfigure}
    \hfill
    \caption{\textbf{Ablations on the (a)-(b) Number of Effective Samples and the (c) Segment Length}. 
    \textbf{(a)-(b)} Effect of Number of Effective Samples for \textbf{(a)} 500 Preferences and for \textbf{(b)} 50 Preferences aggregated over the \textit{Peg Insert Side}, \textit{Sweep Into}, and \textit{Drawer Open} tasks, each with 5 seeds. The results suggest that for a high number of available data points \textbf{(a)}, a smaller number of effective samples leads to improved results, allowing the policy to deviate more from the data distribution's support. For smaller number of preferences \textbf{(b)}, staying close to the data distribution with a higher relative number of effective samples is beneficial.
    \textbf{(c)} Increasing segment length amplifies the \textit{distribution shift} when using single-timestep advantages for policy optimization, leading to worse performance. This problem is absent when using segment-based advantages for the policy update. Experiments are done on 5 seeds, and the same number of state-action pairs is used for all segment lengths.
    }
    \vspace{-2mm}
    \label{fig:ablation_combined}
\end{figure*}

\textbf{Segment- and State-action-based Updates.}
In addition to the results reported in \cref{tab:sr_metaworld_transposed}, we explicitly compare segment-level policy updates with state-action-level updates that rely on the learned advantage function. This comparison directly evaluates the effect of the training and inference distribution shift discussed earlier: while the advantage function is trained on segment-level preference data, state-action-based updates require querying it on individual time steps, thereby reintroducing the mismatch present in standard PbRL methods. We further consider two architectures for the advantage function, parameterized either by an MLP or a Transformer.
Aggregated results across all tasks are reported in \cref{tab:aggregated_performance}. Consistent with our analysis, PAWS achieves substantially higher performance than policies updated using single-step advantages. Additional per-task results are provided in \cref{tab:tasks_performance_100} and \cref{tab:tasks_performance_1000} in \cref{appdx::Comparison of Segment and State-based updates}.

We note that our setting assumes action quality is correlated over time, reflected by our use of different policies with varying performance for data collection. We argue this is more realistic than datasets that simulate varying data quality by adding Gaussian noise to individual actions~\citep{hejna2024contrastive}. Under uncorrelated noise, within-segment quality would be highly variable and single-action importance weights would likely perform better. In the multi-expert scenario considered here, however, segment-level weights clearly outperform single-step weights.

\textbf{Varying the Segment Length.}
We evaluate our method across different segment lengths and compare segment-wise policy updates with state-action-based updates. When the segment length is one, both settings coincide. In practice, however, such short segments are insufficient for reliable preference assessment, as they provide limited contextual information about behavior quality. For all experiments, we fix the total number of segments to $1{,}000$ and generate preference datasets by uniformly sampling segment pairs. To ensure a fair comparison across segment lengths and maintain an approximately equal number of state-action samples for training the advantage function, we increase the number of sampled preferences as the segment length decreases. For instance, for segments of length $64$, we sample $10{,}000$ preference pairs, while for segments of length $32$, we sample $20{,}000$ preference pairs. 
Results averaged over 10 Meta-World tasks and 5 random seeds are reported in \cref{fig:diff_segment_lengths}. As can be seen, state-action-based updates degrade sharply once the segment length exceeds one, with the largest drop occurring up to a length of roughly $16$ and performance remaining low for longer segments, reflecting the training and inference distribution shift. In contrast, the performance of our method remains substantially more stable, demonstrating that segment-wise policy updates effectively mitigate this degradation and highlighting the benefits of preserving segment-level preference information.

To further isolate the role of segment length at the policy-update stage, we additionally train a Transformer-based advantage function on $500$ preferences of segment length $64$ and then perform policy updates with progressively shorter segments. As shown in \cref{tab:update_segment_length}, the success rate decreases monotonically as the update-time segment length shrinks. This confirms that shorter segments at update time degrade performance even when the advantage function itself is trained on longer, more informative segments, reinforcing the importance of operating at the segment level throughout both training and policy optimization.
\begin{table}[t]
\centering
\caption{The Transformer-based advantage function is trained once on $500$ preferences, then used to update the policy with progressively shorter segments. Averaged over 10 Meta-World tasks and 5 seeds. The $64$ column corresponds to the setting reported in \cref{tab:sr_metaworld_transposed}.}
\label{tab:update_segment_length}
\begin{tabular}{lcccc}
\toprule
\textbf{Segment length} & $64$ (\cref{tab:sr_metaworld_transposed}) & $32$ & $16$ & $8$ \\
\midrule
\textbf{Success Rate [\%]} & $78.2$ & $75.8$ & $69.9$ & $64.3$ \\
\bottomrule
\end{tabular}
\end{table}

\textbf{Spearman’s Rank Correlation for Segment- and State-action-based Updates.}
To assess temporal credit assignment, we measure how well learned policies preserve the expert’s relative ranking of trajectory segments using Spearman’s rank correlation coefficient $r_s$. This metric computes the Pearson coefficient for the rankings of two vectors. More details are presented in \cref{sec:spearman}. In our case, for each task, we compute the action likelihoods at each state under three policies: the expert SAC policy, the policy trained using our segment-based updates, and the policy trained using single-step updates.

For each trajectory segment, we aggregate the corresponding likelihoods and compute Spearman’s rank correlation between the expert policy and each learned policy. This yields two correlation values per segment, one for the segment-based update and one for the state-action-based update. Higher correlation indicates better preservation of the expert’s preference ordering and, consequently, more effective temporal credit assignment.
We evaluate all tasks using 5 random seeds and 500 preference queries. Overall, the segment-based update used in PAWS achieves a higher mean Spearman’s rank correlation coefficient ($r_s^{seg.}=0.22$) than the step-based approach ($r_s^{step}=0.05$), indicating that segment-level updates better preserve the expert’s ranking over behaviors. Per-task results are reported in \cref{sec:spearman}.

\section{Conclusion}
We propose PAWS, a novel method for preference learning that first learns an advantage function and subsequently updates the policy by exploiting it. Notably, PAWS updates the policy based on segment data rather than single state-action pairs to mitigate the \textit{temporal credit assignment problem} that arises because the advantage function is learned on whole segments rather than single state-action pairs. Overall, the proposed method achieves superior performance on various simulated manipulation and locomotion tasks, with preferences coming from both a simulated oracle and real humans. Our results further indicate that the gains stem primarily from aligning the training and inference distributions of the learned utility, rather than from the specific architectural choice of the advantage function. Both the MLP and Transformer variants substantially outperform their per-step counterparts, and the benefit is particularly pronounced in the low-data regime, where step-level utility estimates are noisiest.

\textbf{Limitations.} PAWS assigns the same importance weight to every state-action pair within a segment (\cref{eq::grad_ml_loss_sa}), which is well suited to the realistic setting of temporally correlated behavior quality but less effective when a segment mixes high- and low-quality actions, where finer-grained per-step weighting would be preferable.
A second limitation is the need to choose the effective sample size $n^*_{\mathrm{eff}}$. While it is more intuitive than directly tuning the KL bound $\epsilon$, our ablation suggests that the optimal \emph{relative} value depends on dataset size, with smaller datasets favoring more conservative updates and larger datasets benefiting from more aggressive updates that selectively focus on high-advantage segments.
Finally, our evaluation is restricted to simulated robotic tasks with oracle and small-scale human preference data. Extending the analysis to large-scale, noisier human feedback remains an important next step.

\textbf{Future Work.} A natural direction is to relax the uniform per-segment weighting by introducing a learned per-step modulation that respects the segment-level training distribution, for example through attention-based reweighting within a segment.
We plan to extend PAWS to the offline-to-online case, where newly collected preference data continuously refines both the advantage function and the policy update, enabling preference-based fine-tuning of pretrained models.
A third direction is to replace the fixed relative $n^*_{\mathrm{eff}}$ with an absolute, data-adaptive criterion, as suggested by the trend observed in \cref{fig:ess_1000} and \cref{fig:ess_100}.
A further direction is to study PAWS under realistic preference noise, including non-stationary annotator behavior. Combining segment-level updates with active query selection could focus limited annotator effort on the segments whose advantage estimates would benefit most, improving sample efficiency.

Beyond robotics, the segment-level update perspective directly maps to preference fine-tuning of large generative models, where comparisons are typically expressed over entire responses rather than individual tokens. Standard RLHF pipelines train a per-token reward model from response-level preferences and then optimize the policy with per-token PPO, recreating the same training and inference mismatch we identify here.

\section*{Impact Statement}

% Authors are \textbf{required} to include a statement of the potential broader
% impact of their work, including its ethical aspects and future societal
% consequences. This statement should be in an unnumbered section at the end of
% the paper (co-located with Acknowledgements -- the two may appear in either
% order, but both must be before References), and does not count toward the paper
% page limit. In many cases, where the ethical impacts and expected societal
% implications are those that are well established when advancing the field of
% Machine Learning, substantial discussion is not required, and a simple
% statement such as the following will suffice:

This paper proposes a novel preference-based reinforcement learning method evaluated on robotic manipulation and locomotion tasks in simulation. Such methods may reduce the need for manually designed reward functions, but can inherit biases present in preference signals provided by humans. All experiments are conducted in simulated environments, and we do not anticipate immediate negative societal impacts from this work.

% The above statement can be used verbatim in such cases, but we encourage
% authors to think about whether there is content which does warrant further
% discussion, as this statement will be apparent if the paper is later flagged
% for ethics review.

\section*{Acknowledgements}
We thank the anonymous reviewers for their valuable feedback and suggestions. The work has been financially supported by the German Research Foundation (DFG, Deutsche Forschungsgemeinschaft) as part of the SFB-1574 -- 471687386 ``Circular Factory'', and by the European Research Council (ERC) under the European Union’s Horizon Europe programme through the project SMARTI³ (Grant Agreement No.~101171393). This work has been supported by the German Federal Ministry of Research, Technology, and Space (BMFTR) under the Robotics Institute Germany (RIG). The authors acknowledge support from the InnovationCampus Future Mobility ICM,  and the state of Baden-Württemberg through bwHPC, as well as the HoreKa supercomputer, funded by the Ministry of Science, Research and the Arts Baden-Württemberg and the German Federal Ministry of Education and Research.


\bibliography{references}

@inproceedings{lee2021bpref,
  title     = {B-Pref: Benchmarking Preference-Based Reinforcement Learning},
  author    = {Lee, Kimin and Smith, Laura and Dragan, Anca D. and Abbeel, Pieter},
  booktitle = {NeurIPS Datasets and Benchmarks Track},
  year      = {2021}
}

@article{Wirth2017,
  title   = {A Survey of Preference-Based Reinforcement Learning Methods},
  author  = {Christian Wirth and Riad Akrour and Gerhard Neumann and Johannes F{\"u}rnkranz},
  journal = {Journal of Machine Learning Research},
  volume  = {18},
  number  = {136},
  pages   = {1--46},
  year    = {2017}
}

@inproceedings{rafailov2023direct,
  title     = {Direct Preference Optimization: Your Language Model is Secretly a Reward Model},
  author    = {Rafael Rafailov and Archit Sharma and Eric Mitchell and Stefano Ermon and Christopher D. Manning and Chelsea Finn},
  booktitle = {Proceedings of the 37th International Conference on Neural Information Processing Systems},
  year      = {2023},
  pages     = {16203--16220},
  doi       = {10.5555/3666122.3668460}
}

@inproceedings{hejna2024contrastive,
  title     = {Contrastive Preference Learning: Learning from Human Feedback without Reinforcement Learning},
  author    = {Joey Hejna and Rafael Rafailov and Harshit Sikchi and Chelsea Finn and Scott Niekum and W. Bradley Knox and Dorsa Sadigh},
  booktitle = {International Conference on Learning Representations (ICLR)},
  year      = {2024}
}

@inproceedings{xu2024dpo,
  title     = {Is DPO Superior to PPO for LLM Alignment? A Comprehensive Study},
  author    = {Xu, Shusheng and Fu, Wei and Gao, Jiaxuan and others},
  booktitle = {Proceedings of the 41st International Conference on Machine Learning},
  year      = {2024}
}

@inproceedings{Christiano2017DeepRL,
  title     = {Deep Reinforcement Learning from Human Preferences},
  author    = {Christiano, Paul F. and Leike, Jan and Brown, Tom B. and Martic, Miljan and Legg, Shane and Amodei, Dario},
  booktitle = {Advances in Neural Information Processing Systems},
  volume    = {30},
  year      = {2017}
}

@inproceedings{stiennon2020learning,
  title     = {Learning to Summarize from Human Feedback},
  author    = {Stiennon, Nisan and Ouyang, Long and Wu, Jeffrey and Ziegler, Daniel and Lowe, Ryan and Voss, Chelsea and Radford, Alec and Amodei, Dario and Christiano, Paul},
  booktitle = {Proceedings of the 34th International Conference on Neural Information Processing Systems},
  year      = {2020},
  volume    = {33},
  pages     = {1--15},
  doi       = {10.48550/arXiv.2009.01325}
}

@article{ziegler2019finetuning,
  title   = {Fine-Tuning Language Models from Human Preferences},
  author  = {Ziegler, Daniel M. and Stiennon, Nisan and Wu, Jeffrey and Brown, Tom B. and Radford, Alec and Amodei, Dario and Christiano, Paul and Irving, Geoffrey},
  journal = {arXiv preprint arXiv:1909.08593},
  year    = {2019}
}

@inproceedings{kim2023preference,
  title     = {Preference Transformer: Modeling Human Preferences using Transformers for {RL}},
  author    = {Changyeon Kim and Jongjin Park and Jinwoo Shin and Honglak Lee and Pieter Abbeel and Kimin Lee},
  booktitle = {International Conference on Learning Representations},
  year      = {2023}
}

@article{deepseek_r1_2025,
  title   = {DeepSeek-R1 incentivizes reasoning in LLMs through reinforcement learning},
  author  = {Guo,Daya and Yang, Dejian and Zhang, Haowei and others},
  journal = {Nature},
  year    = {2025}
}

@inproceedings{Abdolmaleki2025Learning,
  title     = {Learning from negative feedback, or positive feedback or both},
  author    = {Abbas Abdolmaleki and Bilal Piot and Bobak Shahriari and Jost Tobias Springenberg and Tim Hertweck and Michael Bloesch and Rishabh Joshi and Thomas Lampe and Junhyuk Oh and Nicolas Heess and Jonas Buchli and Martin Riedmiller},
  booktitle = {International Conference on Learning Representations},
  year      = {2025}
}

@article{knox2024models,
  title   = {Models of human preference for learning reward functions},
  author  = {Knox, W. Bradley and Hatgis-Kessell, Stephane and Booth, Serena and Niekum, Scott and Stone, Peter and Allievi, Alessandro G},
  journal = {Transactions on Machine Learning Research},
  year    = {2024}
}

@inproceedings{peters2010relative,
  title     = {Relative entropy policy search},
  author    = {Peters, Jan and Mulling, Katharina and Altun, Yasemin},
  booktitle = {Proceedings of the AAAI Conference on Artificial Intelligence},
  volume    = {24},
  number    = {1},
  pages     = {1607--1612},
  year      = {2010}
}

@article{daniel2016hierarchical,
  title   = {Hierarchical relative entropy policy search},
  author  = {Daniel, Christian and Neumann, Gerhard and Kroemer, Oliver and Peters, Jan},
  journal = {Journal of Machine Learning Research},
  volume  = {17},
  number  = {93},
  pages   = {1--50},
  year    = {2016}
}

@article{neumann2008fitted,
  title   = {Fitted q-iteration by advantage weighted regression},
  author  = {Neumann, Gerhard and Peters, Jan},
  journal = {Advances in neural information processing systems},
  volume  = {21},
  year    = {2008}
}

@article{nair2020awac,
  title   = {Awac: Accelerating online reinforcement learning with offline datasets},
  author  = {Nair, Ashvin and Gupta, Abhishek and Dalal, Murtaza and Levine, Sergey},
  journal = {arXiv preprint arXiv:2006.09359},
  year    = {2020}
}

@book{boyd2004convex,
  title     = {Convex optimization},
  author    = {Boyd, Stephen P and Vandenberghe, Lieven},
  year      = {2004},
  publisher = {Cambridge university press}
}

@inproceedings{abdolmaleki2018maximum,
  title     = {Maximum a Posteriori Policy Optimisation},
  author    = {Abdolmaleki, Abbas and Springenberg, Jost Tobias and Tassa, Yuval and Munos, Remi and Heess, Nicolas and Riedmiller, Martin},
  booktitle = {International Conference on Learning Representations},
  year      = {2018}
}

@inproceedings{peters2007reinforcement,
  title     = {Reinforcement learning by reward-weighted regression for operational space control},
  author    = {Peters, Jan and Schaal, Stefan},
  booktitle = {Proceedings of the 24th international conference on Machine learning},
  pages     = {745--750},
  year      = {2007}
}

@article{kober2008policy,
  title   = {Policy search for motor primitives in robotics},
  author  = {Kober, Jens and Peters, Jan},
  journal = {Advances in neural information processing systems},
  volume  = {21},
  year    = {2008}
}

@article{abdolmaleki2015model,
  title   = {Model-based relative entropy stochastic search},
  author  = {Abdolmaleki, Abbas and Lioutikov, Rudolf and Peters, Jan R and Lau, Nuno and Paulo Reis, Luis and Neumann, Gerhard},
  journal = {Advances in Neural Information Processing Systems},
  volume  = {28},
  year    = {2015}
}

@inproceedings{instruct_gpt,
  title     = {Training language models to follow instructions with human feedback},
  author    = {Ouyang, Long and Wu, Jeff and Jiang, Xu and Almeida, Diogo and Wainwright, Carroll L. and Mishkin, Pamela and Zhang, Chong and Agarwal, Sandhini and Slama, Katarina and Ray, Alex and Schulman, John and Hilton, Jacob and Kelton, Fraser and Miller, Luke and Simens, Maddie and Askell, Amanda and Welbl, Johannes and Leike, Jan and Lowe, Ryan and Christiano, Paul},
  booktitle = {Advances in Neural Information Processing Systems (NeurIPS)},
  year      = {2022}
}

@article{schulman2017proximal,
  title   = {Proximal Policy Optimization Algorithms},
  author  = {Schulman, John and Wolski, Filip and Dhariwal, Prafulla and Radford, Alec and Klimov, Oleg},
  journal = {arXiv preprint arXiv:1707.06347},
  year    = {2017}
}

@inproceedings{pebble,
  title     = {PEBBLE: Feedback-Efficient Interactive Reinforcement Learning via Relabeling Experience and Unsupervised Pre-training},
  author    = {Lee, Kimin and Smith, Laura and Abbeel, Pieter},
  booktitle = {International Conference on Machine Learning (ICML)},
  year      = {2021}
}

@article{kostrikov2021offline,
  title   = {Offline reinforcement learning with implicit q-learning},
  author  = {Kostrikov, Ilya and Nair, Ashvin and Levine, Sergey},
  journal = {arXiv preprint arXiv:2110.06169},
  year    = {2021}
}

@inproceedings{schulman2015trust,
  title     = {Trust Region Policy Optimization},
  author    = {Schulman, John and Levine, Sergey and Abbeel, Pieter and Jordan, Michael I. and Moritz, Philipp},
  booktitle = {Proceedings of the 32nd International Conference on Machine Learning (ICML)},
  series    = {Proceedings of Machine Learning Research},
  volume    = {37},
  pages     = {1889--1897},
  year      = {2015},
  publisher = {PMLR}
}

@inproceedings{li2024open,
  title     = {Open the Black Box: Step-based Policy Updates for Temporally-Correlated Episodic Reinforcement Learning},
  author    = {Li, Ge and Zhou, Hongyi and Roth, Dominik and Thilges, Serge and Otto, Fabian and Lioutikov, Rudolf and Neumann, Gerhard},
  booktitle = {The Twelfth International Conference on Learning Representations},
  year      = {2024}
}

@inproceedings{li2025top,
  title     = {TOP-ERL: Transformer-based Off-Policy Episodic Reinforcement Learning},
  author    = {Li, Ge and Tian, Dong and Zhou, Hongyi and Jiang, Xinkai and Lioutikov, Rudolf and Neumann, Gerhard},
  booktitle = {The Thirteenth International Conference on Learning Representations},
  year      = {2025}
}

@inproceedings{hoang2025geometryaware,
  title     = {Geometry-aware {RL} for Manipulation of Varying Shapes and Deformable Objects},
  author    = {Tai Hoang and Huy Le and Philipp Becker and Vien Anh Ngo and Gerhard Neumann},
  booktitle = {The Thirteenth International Conference on Learning Representations},
  year      = {2025}
}

@article{BradleyTerry1952,
  author  = {Bradley, Ralph Allan and Terry, Milton E.},
  title   = {Rank Analysis of Incomplete Block Designs: I. The Method of Paired Comparisons},
  journal = {Biometrika},
  year    = {1952},
  volume  = {39},
  number  = {3/4},
  pages   = {324--345}
}

@article{baker2009action,
  title     = {Action understanding as inverse planning},
  author    = {Baker, Chris L and Saxe, Rebecca and Tenenbaum, Joshua B},
  journal   = {Cognition},
  volume    = {113},
  number    = {3},
  pages     = {329--349},
  year      = {2009},
  publisher = {Elsevier}
}

@article{ivison2024unpacking,
  title   = {Unpacking {DPO} and {PPO}: Disentangling Best Practices for Learning from Preference Feedback},
  author  = {Ivison, Hamish and Wang, Yizhong and Liu, Jiacheng and Wu, Zeqiu and Pyatkin, Valentina and Lambert, Nathan and Smith, Noah A and Choi, Yejin and Hajishirzi, Hanna},
  journal = {Advances in Neural Information Processing Systems},
  volume  = {37},
  pages   = {36602--36633},
  year    = {2024}
}

@inproceedings{ng2000algorithms,
  title     = {Algorithms for inverse reinforcement learning.},
  author    = {Ng, Andrew Y and Russell, Stuart and others},
  booktitle = {ICML},
  volume    = {1},
  number    = {2},
  pages     = {2},
  year      = {2000}
}

@article{nakano2021webgpt,
  title   = {Webgpt: Browser-assisted question-answering with human feedback},
  author  = {Nakano, Reiichiro and Hilton, Jacob and Balaji, Suchir and Wu, Jeff and Ouyang, Long and Kim, Christina and Hesse, Christopher and Jain, Shantanu and Kosaraju, Vineet and Saunders, William and others},
  journal = {arXiv preprint arXiv:2112.09332},
  year    = {2021}
}

@inproceedings{gao2023scaling,
  title        = {Scaling laws for reward model overoptimization},
  author       = {Gao, Leo and Schulman, John and Hilton, Jacob},
  booktitle    = {International Conference on Machine Learning},
  pages        = {10835--10866},
  year         = {2023},
  organization = {PMLR}
}

@article{zhao2024prefmmt,
  title   = {Prefmmt: Modeling human preferences in preference-based reinforcement learning with multimodal transformers},
  author  = {Zhao, Dezhong and Wang, Ruiqi and Suh, Dayoon and Kim, Taehyeon and Yuan, Ziqin and Min, Byung-Cheol and Chen, Guohua},
  journal = {arXiv preprint arXiv:2409.13683},
  year    = {2024}
}

@inproceedings{yu2020meta,
  title     = {Meta-World: A Benchmark and Evaluation for Multi-Task and Meta Reinforcement Learning},
  author    = {Yu, Tianhe and Quillen, Deirdre and He, Zhanpeng and Julian, Ryan and Hausman, Karol and Finn, Chelsea and Levine, Sergey},
  booktitle = {Conference on Robot Learning},
  pages     = {1094--1100},
  year      = {2020},
  volume    = {100},
  series    = {Proceedings of Machine Learning Research}
}

@inproceedings{vaswani2017attention,
  title     = {Attention Is All You Need},
  author    = {Vaswani, Ashish and Shazeer, Noam and Parmar, Niki and Uszkoreit, Jakob and Jones, Llion and Gomez, Aidan N. and Kaiser, {\L}ukasz and Polosukhin, Illia},
  booktitle = {Advances in Neural Information Processing Systems},
  pages     = {5998--6008},
  year      = {2017}
}

@inproceedings{cho2025ppl,
  title     = {Policy-labeled Preference Learning: Is Preference Enough for RLHF?},
  author    = {Taehyun Cho and Seokhun Ju and Seungyub Han and Dohyeong Kim and Kyungjae Lee and Jungwoo Lee},
  booktitle = {Proceedings of the 42nd International Conference on Machine Learning (ICML)},
  year      = {2025}
}

@inproceedings{haarnoja2018soft,
  title     = {Soft Actor-Critic: Off-Policy Maximum Entropy Deep Reinforcement Learning with a Stochastic Actor},
  author    = {Tuomas Haarnoja and Aurick Zhou and Pieter Abbeel and Sergey Levine},
  booktitle = {Proceedings of the 35th International Conference on Machine Learning (ICML)},
  year      = {2018},
  pages     = {1861--1870},
  volume    = {80},
  series    = {Proceedings of Machine Learning Research}
}

@inproceedings{taranovic2023adversarial,
  title     = {Adversarial Imitation Learning with Preferences},
  author    = {Taranovic, Aleksandar and Kupcsik, Andras G. and Freymuth, Niklas and Neumann, Gerhard},
  booktitle = {International Conference on Learning Representations (ICLR)},
  year      = {2023}
}

@article{biyik2021learning,
  title   = {Learning Reward Functions from Diverse Sources of Human Feedback: Optimally Integrating Demonstrations and Preferences},
  author  = {Biyik, Erdem and Losey, Dylan P. and Palan, Malayandi and Landolfi, Nicholas C. and Shevchuk, Gleb and Sadigh, Dorsa},
  journal = {The International Journal of Robotics Research (IJRR)},
  year    = {2021},
  doi     = {10.1177/02783649211041652}
}

@inproceedings{ibarz2018reward,
  title     = {Reward learning from human preferences and demonstrations in Atari},
  author    = {Ibarz, Borja and Leike, Jan and Pohlen, Tobias and Irving, Geoffrey and Legg, Shane and Amodei, Dario},
  booktitle = {Advances in Neural Information Processing Systems (NeurIPS)},
  year      = {2018}
}

@inproceedings{myers2021learning,
  title     = {Learning Multimodal Rewards from Rankings},
  author    = {Myers, Vivek and Biyik, Erdem and Anari, Nima and Sadigh, Dorsa},
  booktitle = {Proceedings of the 5th Conference on Robot Learning (CoRL)},
  year      = {2021}
}

@inproceedings{an2023direct,
  title     = {Direct Preference-based Policy Optimization without Reward Modeling},
  author    = {Gaon An and Junhyeok Lee and Xingdong Zuo and Norio Kosaka and Kyung-Min Kim and Hyun Oh Song},
  booktitle = {Advances in Neural Information Processing Systems},
  year      = {2023}
}

@inproceedings{hejna2023inverse,
  title     = {Inverse Preference Learning: Preference-based RL without a Reward Function},
  author    = {Joey Hejna and Dorsa Sadigh},
  booktitle = {36th Conference on Neural Information Processing Systems (NeurIPS 2023)},
  year      = {2023}
}

@article{gao2024hindsight,
  title   = {Hindsight Preference Learning for Offline Preference-based Reinforcement Learning},
  author  = {Chen-Xiao Gao and Shengjun Fang and Chenjun Xiao and Yang Yu and Zongzhang Zhang},
  journal = {arXiv preprint arXiv:2407.04451},
  year    = {2024}
}

@inproceedings{verma2024hindsight,
  title     = {Hindsight {\scshape PRIOR}s for Reward Learning from Human Preferences},
  author    = {Mudit Verma and Katherine Metcalf},
  booktitle = {12th International Conference on Learning Representations (ICLR 2024)},
  year      = {2024}
}

@article{towers2024gymnasium,
  title   = {Gymnasium: A Standard Interface for Reinforcement Learning Environments},
  author  = {Towers, Mark and Kwiatkowski, Ariel and Terry, Jordan and Balis, John U and De Cola, Gianluca and Deleu, Tristan and Goul{\~a}o, Manuel and Kallinteris, Andreas and Krimmel, Markus and KG, Arjun and others},
  journal = {arXiv preprint arXiv:2407.17032},
  year    = {2024}
}

@inproceedings{haarnoja2017reinforcement,
  title        = {Reinforcement learning with deep energy-based policies},
  author       = {Haarnoja, Tuomas and Tang, Haoran and Abbeel, Pieter and Levine, Sergey},
  booktitle    = {International conference on machine learning},
  pages        = {1352--1361},
  year         = {2017},
  organization = {PMLR}
}

@article{fu2020d4rl,
  title   = {D4rl: Datasets for deep data-driven reinforcement learning},
  author  = {Fu, Justin and Kumar, Aviral and Nachum, Ofir and Tucker, George and Levine, Sergey},
  journal = {arXiv preprint arXiv:2004.07219},
  year    = {2020}
}

@inproceedings{otto_iclr2021,
  title={Differentiable Trust Region Layers for Deep Reinforcement Learning},
  author={Otto, Fabian and Becker, Philipp and Anh Vien, Ngo and Ziesche, Hanna Carolin and Neumann, Gerhard},
  booktitle={International Conference on Learning Representations},
  year={2021}
}
\bibliographystyle{icml2026}

%%%%%%%%%%%%%%%%%%%%%%%%%%%%%%%%%%%%%%%%%%%%%%%%%%%%%%%%%%%%%%%%%%%%%%%%%%%%%%%
%%%%%%%%%%%%%%%%%%%%%%%%%%%%%%%%%%%%%%%%%%%%%%%%%%%%%%%%%%%%%%%%%%%%%%%%%%%%%%%
% APPENDIX
%%%%%%%%%%%%%%%%%%%%%%%%%%%%%%%%%%%%%%%%%%%%%%%%%%%%%%%%%%%%%%%%%%%%%%%%%%%%%%%
%%%%%%%%%%%%%%%%%%%%%%%%%%%%%%%%%%%%%%%%%%%%%%%%%%%%%%%%%%%%%%%%%%%%%%%%%%%%%%%
\newpage
\appendix
\onecolumn

\section{Proofs}
\subsection{Proof of the Optimal Distribution}\label{appdx::derivation_opt_distr}

\begin{proof}
The Lagrangian to the optimization problem in~\cref{eq::PolicyOpt} is given by
\begin{align} 
    L(p,\lambda,\beta) &= \int p(\tau) A_\phi(\tau)\,d\tau + \lambda\left(\epsilon - \int p(\tau)\bigl(\log p(\tau) - \log p_D(\tau)\bigr) d\tau\right) \nonumber \\
    &+ \beta \left(1- \int p(\tau) \,d\tau \right) \nonumber \\
    &=\int p(\tau) \Bigl( A_\phi(\tau) -\lambda \bigl(\log p(\tau)-\log p_D(\tau)\bigr)-\beta \Bigr) d\tau + \lambda \epsilon +\beta. \label{appdx::eq::lagrangian}
\end{align}

Obtaining the optimal solution for fixed Lagrange multipliers $\lambda, \beta$ yields
\begin{align}
    0 &\overset{!}{=} \nabla_p L =A_\phi(\tau) - \lambda\log p(\tau) + \lambda\log p_D(\tau) - \lambda - \beta  \nonumber \\
    &\rightarrow p^* = \exp\left(\frac{A_\phi(\tau) + \lambda \log p_D(\tau)}{\lambda}\right) \exp\left(\frac{-\beta-\lambda}{\lambda}\right) \label{appdx::eq::opt_pol}
\end{align}
\end{proof}

\subsection{Proof of the Dual Function}\label{appdx::derivation_dual_fct}

\begin{proof}    
We can eliminate the Lagrange multiplier $\beta$ to the normalization constraint 
\begin{align}\label{appdx:eq::norm_constraint}
    \int p(\tau) \,d\tau= 1
\end{align} 
in the optimization problem~\cref{eq::PolicyOpt} by inserting the optimal solution in~\cref{appdx::eq::opt_pol} into the~\cref{appdx:eq::norm_constraint}, such that we obtain
\begin{align}
    1&=\exp\left(\frac{-\beta-\lambda}{\lambda}\right)\int\exp\left(\frac{A_\phi(\tau)+\lambda\log p_D(\tau)}{\lambda}\right)d\tau \nonumber \\
    0&=\frac{-\beta-\lambda}{\lambda} + \log \int\exp{\left(\frac{A_\phi(\tau)+\lambda\log p_D(\tau)}{\lambda}\right)}d\tau \nonumber \\
    \beta &= -\lambda +\lambda\log \int\exp{\left(\frac{A_\phi(\tau)+\lambda\log p_D(\tau)}{\lambda}\right)}d\tau \label{appdx::eq::opt_norm_lm}
\end{align}
By inserting the optimal solution $p^*$ in~\cref{appdx::eq::opt_pol} and $\beta$ in~\cref{appdx::eq::opt_norm_lm} into the Lagrangian in~\cref{appdx::eq::lagrangian}, we obtain
\begin{align}
    g(\lambda) &= \int p^*(\tau)\Bigl(A_\phi(\tau)-A_\phi(\tau)-\lambda\log p_D(\tau)+\beta+\lambda +\lambda\log p_D(\tau)-\beta\Bigr)d\tau +\lambda \epsilon +\beta \nonumber \\
    &= \int p^*(\tau)\lambda \,d\tau + \lambda\epsilon + \beta\nonumber \\
    &= \lambda + \lambda \epsilon +\beta \nonumber \\
    &= \lambda\epsilon +\lambda\log \int\exp{\left(\frac{A_\phi(\tau)+\lambda\log p_D(\tau)}{\lambda}\right)}d\tau \nonumber \\
    &= \lambda\epsilon +\lambda\log \int p_D(\tau)\exp{\left(\frac{A_\phi(\tau)}{\lambda}\right)}d\tau .\label{appdx::eq::dual}
\end{align}
\end{proof}

\pagebreak
\subsection{Proof of Policy Extraction}\label{appdx::derivation_weighted_mle}
\begin{proof}
    The moment-projection of the non-parametric optimal distribution \(p^*\) onto the set of parametrizable distributions,
    \begin{equation*}
        \theta^* = \operatorname{argmin}_{\theta} D_{\mathrm{KL}} \bigl(p^*(\tau) \,\|\, p_\theta(\tau) \bigr),
    \end{equation*}
    is a likelihood maximization as the KL objective,
    \begin{align}
        D_{\mathrm{KL}} \bigl(p^*(\tau) \,\|\, p_\theta(\tau) \bigr) &= \int p^*(\tau) \log \frac{p^*(\tau)}{p_\theta(\tau)} d\tau = \int p_D(\tau) \frac{p^*(\tau)}{p_D(\tau)} \log \frac{p^*(\tau)}{p_\theta(\tau)} d\tau \nonumber \\
        &= \int p_D(\tau) w(\tau) \log p^*(\tau) \,d\tau - \int p_D(\tau) w(\tau) \log p_\theta(\tau) \,d\tau \nonumber \\
        &= C(p_D, p^*) - \mathbb{E}_{p_D} \bigl[w(\tau) \log p_\theta(\tau) \bigr], \nonumber 
    \end{align}
    equals the importance sampled maximum likelihood objective, \(\mathcal{L}(\theta) = \mathbb{E}_{p_D} \bigl[w(\tau) \log p_\theta(\tau) \bigr]\), up to negation and a parameter-independent constant $C(p_D, p^*)$.
    Thus we have 
    \begin{equation}
        \theta^* = \operatorname{argmin}_{\theta} D_{\mathrm{KL}} \bigl(p^*(\tau) \,\|\, p_\theta(\tau) \bigr) = \operatorname{argmax}_{\theta} \mathbb{E}_{p_D} \bigl[w(\tau) \log p_\theta(\tau) \bigr] = \operatorname{argmax}_{\theta} \mathcal{L}(\theta). \nonumber
    \end{equation}
    The importance sampling weights $w(\tau)$ directly follow from~\cref{eq::opt_sol}, yielding
    \begin{equation*}
        w(\tau) = \frac{p^*(\tau)}{p_D(\tau)} \propto \exp\left( \frac{1}{\lambda}A_\phi(\tau)\right),
    \end{equation*}
    or more specifically using~\cref{appdx::eq::opt_pol},
    \begin{align}
        w(\tau) &= \frac{\exp\left(\frac{A_\phi(\tau)}{\lambda}\right) \exp\left(\log p_D(\tau) - \frac{\beta+\lambda}{\lambda}\right)}{p_D(\tau)} \nonumber \\
        &= \exp\left(\frac{A_\phi(\tau)}{\lambda}\right) \exp\left(\frac{-\beta-\lambda}{\lambda}\right) 
        \eqcolon \exp\left(\frac{A_\phi(\tau)}{\lambda}\right) \cdot Z, \label{appdx::eq::impweight_norm}
    \end{align}
    for the normalization constant $Z$.
    Inserting \cref{eq::segment_prob_def} into \cref{eq::ml_loss},
    \begin{align}
        \mathbb{E}_{p_D} \bigl[w(\tau) \log p_\theta(\tau) \bigr]
        &= \mathbb{E}_{p_D} \left[w(\tau) \log \left( p_D(s_0) \prod_{t=0}^{N-1} p(s_{t+1}|s_t, a_t) \pi_{\theta}(a_t|s_t) \right) \right]  \nonumber \\
        &= \mathbb{E}_{p_D} \left[w(\tau) \left( \log p_D(s_0) + \sum_{t=0}^{N-1} \bigl( \log p(s_{t+1}|s_t, a_t) + \log \pi_{\theta}(a_t|s_t) \bigr) \right) \right]  \nonumber \\
        &= \mathbb{E}_{p_D} \left[w(\tau) \sum_{t=0}^{N-1}  \log \pi_{\theta}(a_t|s_t) + C(\tau) \right] \nonumber,
    \end{align}
    considering that $C(\tau)$ is independent of the optimized parameters $\theta$, and that the importance weight normalization constant $Z$ of~\cref{appdx::eq::impweight_norm} is constant, the objective reduces to
    \begin{align}
        \operatorname{argmax}_{\theta} \mathbb{E}_{p_D} \bigl[w(\tau) \log p_\theta(\tau) \bigr]
        &= \operatorname{argmax}_{\theta}\mathbb{E}_{p_D} \left[ \sum_{t=0}^{N-1} \exp\left(\frac{A_\phi(\tau)}{\lambda}\right) \log \pi_{\theta}(a_t|s_t) \right]. \nonumber
    \end{align}
\end{proof}

\pagebreak
\subsection{Derivation of the Gradient}\label{appdx::derivation_state_action_grad}
\begin{proof}
We can rewrite the more common state-action policy optimization counterpart to the maximum likelihood objective from~\cref{eq::ml_loss} as 
\begin{align}\label{eq::ml_loss_sa}
    \mathcal{L}_{ML}(\theta) = \mathbb{E}_{p_D}\left[\frac{1}{Z}\exp\left(\frac{1}{\lambda}\sum_tA_\phi(s_t,a_t)\right)\left(\log p(s_0) + \sum_t \bigl(\log\pi_\theta(a_t|s_t) +\log p(s_{t+1}|s_t,a_t)\bigr)\right)\right],
\end{align}
which includes the unknown initial state distribution $p(s_0)$ and the transition dynamics $p(s_{t+1}|s_t,a_t)$ as well as the intractable normalization constant \(Z = \int \exp\left(\frac{1}{\lambda}\sum_t A_\phi(s_t,a_t)\right) \,d\tau\). However, note that neither of these entities affects the optimization, as the $\theta$-independent log-terms vanish when calculating the gradient and the constant factor $1/Z$ does not change the optimum, yielding
\begin{align}
    \nabla_\theta \mathcal{L} = \mathbb{E}_{p_D}\left[\exp\left(\frac{1}{\lambda}\sum_t A_\phi(s_t,a_t)\right)\left(\sum_t \nabla_\theta\log\pi_\theta(a_t|s_t)\right)\right],
\end{align}
which is closely related to \citet{kober2008policy}.
\end{proof}
\section{Average Returns}
\label{appdx:Average returns}
In \cref{tab:transposed_performance}, we present the results of mean episode returns corresponding to the experiment setup of \cref{tab:sr_metaworld_transposed}.
The relative performance ordering per task remains largely unchanged when compared to the ranking based on success rate in \cref{tab:sr_metaworld_transposed}.

\begin{table}[h]
\centering
\small
\setlength\tabcolsep{2pt}
\caption{Mean episode returns across methods and preference counts. Best results for $n=50$ and $n=500$ are highlighted with \textbf{\textcolor{groupone}{orange}} and \textbf{\textcolor{grouptwo}{teal}}, respectively. }
\label{tab:transposed_performance}
\begin{adjustbox}{max width=\textwidth}
\small
\renewcommand{\arraystretch}{1.4}
\begin{tabular}{lcccccccccccccccc}
\toprule
\multicolumn{1}{c}{} & \multicolumn{16}{c}{\textbf{Methods}} \\
\cmidrule(lr){2-17}
\textbf{Task} & \multicolumn{2}{c}{\textbf{BC}} & \multicolumn{2}{c}{\textbf{P-IQL}} & \multicolumn{2}{c}{\textbf{CPL}} & \multicolumn{2}{c}{\textbf{CPL+KL}} & \multicolumn{2}{c}{\textbf{Pref Trans.}} & \multicolumn{2}{c}{\textbf{IPL}} & \multicolumn{2}{c}{\textbf{PAWS (Trans.)}} & \multicolumn{2}{c}{\textbf{PAWS (MLP)}} \\
\cmidrule(lr){2-3}\cmidrule(lr){4-5}\cmidrule(lr){6-7}\cmidrule(lr){8-9}\cmidrule(lr){10-11}\cmidrule(lr){12-13}\cmidrule(lr){14-15}\cmidrule(lr){16-17}
 & \textbf{50} & \textbf{500} & \textbf{50} & \textbf{500} & \textbf{50} & \textbf{500} & \textbf{50} & \textbf{500} & \textbf{50} & \textbf{500} & \textbf{50} & \textbf{500} & \textbf{50} & \textbf{500} & \textbf{50} & \textbf{500} \\
\midrule
Button Press & 1423$\pm$37 & 1593$\pm$40 & 1423$\pm$39 & 1590$\pm$53 & 1443$\pm$86 & 1598$\pm$55 & 1416$\pm$74 & 1581$\pm$29 & 1483$\pm$97 & 1626$\pm$30 & 1339$\pm$82 & 1584$\pm$33 & 1356$\pm$74 & 1568$\pm$25 & \bestgroupone{1583$\pm$81} & \bestgrouptwo{1722$\pm$22} \\
Door Open & 870$\pm$90 & 970$\pm$38 & 870$\pm$95 & 1593$\pm$43 & 970$\pm$126 & 1598$\pm$55 & 762$\pm$125 & 1581$\pm$29 & 1136$\pm$151 & 1527$\pm$57 & 890$\pm$98 & 1047$\pm$51 & 822$\pm$127 & 1568$\pm$25 & \bestgroupone{1294$\pm$158} & \bestgrouptwo{1819$\pm$27} \\
Drawer Open & 1508$\pm$59 & 1562$\pm$22 & \bestgroupone{1508$\pm$80} & 1562$\pm$22 & 1463$\pm$78 & 1626$\pm$44 & 1472$\pm$78 & \bestgrouptwo{1646$\pm$41} & 1255$\pm$82 & 1631$\pm$22 & 1259$\pm$53 & 1510$\pm$31 & 1481$\pm$61 & 1617$\pm$35 & 1104$\pm$194 & 1634$\pm$47 \\
Faucet Close & 1516$\pm$59 & 1692$\pm$31 & 1516$\pm$59 & 1692$\pm$31 & 1599$\pm$79 & 1900$\pm$27 & 1516$\pm$86 & 1662$\pm$61 & 1581$\pm$69 & 1935$\pm$46 & 1499$\pm$78 & 1713$\pm$49 & 1449$\pm$68 & 1661$\pm$64 & \bestgroupone{1696$\pm$133} & \bestgrouptwo{1997$\pm$37} \\
Lever Pull & 366$\pm$43 & 426$\pm$18 & \bestgroupone{366$\pm$45} & 426$\pm$18 & 343$\pm$33 & 406$\pm$32 & 347$\pm$32 & 450$\pm$51 & 331$\pm$49 & 426$\pm$42 & 343$\pm$24 & 439$\pm$20 & 338$\pm$33 & \bestgrouptwo{467$\pm$45} & 350$\pm$57 & 462$\pm$39 \\
Peg Insert Side & 882$\pm$87 & 1281$\pm$28 & \bestgroupone{882$\pm$87} & 1281$\pm$28 & 670$\pm$107 & 1511$\pm$49 & 836$\pm$98 & 1426$\pm$74 & 708$\pm$109 & 1508$\pm$20 & 731$\pm$119 & 1230$\pm$32 & 814$\pm$125 & 1463$\pm$65 & 505$\pm$136 & \bestgrouptwo{1545$\pm$47} \\
Plate Slide & 1049$\pm$85 & 1156$\pm$94 & 1049$\pm$89 & 1156$\pm$97 & \bestgroupone{1095$\pm$78} & \bestgrouptwo{1582$\pm$78} & 950$\pm$124 & 1395$\pm$108 & 1150$\pm$120 & 1617$\pm$62 & 1140$\pm$125 & 1218$\pm$92 & 955$\pm$104 & 1344$\pm$91 & 1070$\pm$130 & 1548$\pm$91 \\
Push Back & 312$\pm$56 & 484$\pm$39 & 312$\pm$56 & 484$\pm$39 & 270$\pm$65 & 607$\pm$111 & 261$\pm$88 & 629$\pm$45 & 262$\pm$53 & 698$\pm$62 & 300$\pm$58 & 490$\pm$47 & 203$\pm$68 & 584$\pm$55 & \bestgroupone{363$\pm$73} & \bestgrouptwo{872$\pm$70} \\
Sweep Into & 410$\pm$85 & 1001$\pm$54 & 410$\pm$90 & 1001$\pm$57 & \bestgroupone{475$\pm$85} & 1117$\pm$103 & 325$\pm$123 & 1142$\pm$100 & 397$\pm$57 & 1109$\pm$59 & 410$\pm$77 & 900$\pm$54 & 439$\pm$111 & 1128$\pm$88 & 463$\pm$99 & \bestgrouptwo{1326$\pm$85} \\
Window Close & 1095$\pm$102 & 1371$\pm$22 & 1095$\pm$97 & 1371$\pm$34 & 1198$\pm$104 & 1530$\pm$34 & 1015$\pm$106 & 1371$\pm$22 & 1107$\pm$229 & 1492$\pm$46 & 1034$\pm$109 & 1390$\pm$36 & 982$\pm$123 & 1358$\pm$47 & \bestgroupone{1373$\pm$90} & \bestgrouptwo{1572$\pm$17} \\
\bottomrule
\end{tabular}
\end{adjustbox}
\end{table}

\newpage
\section{Comparison of Segment- and State-Based Updates}
\label{appdx::Comparison of Segment and State-based updates}
The detailed per-task success rates for the policy update type and architecture ablation can be found in \cref{tab:tasks_performance_100} and \cref{tab:tasks_performance_1000}.
PAWS consistently outperforms the corresponding state-wise update for the same architecture with only few exceptions.
In the $n=50$ preferences case and applied to the Transformer network, a slightly lower success rate for \emph{Lever Pull} and a tie for \emph{Sweep Into} is achieved.
When more preferences, $n=500$, are available, only slight variations in success rate for \emph{Lever Pull } break the trend.

\vspace{0.5cm}

\begin{table}[h]
\centering
\small
\setlength\tabcolsep{6pt}
\caption{Performance results for individual tasks with \textbf{50} preferences. Values represent success rates (\%) $\pm$ 2SE. Best results are highlighted with \textbf{bold}.}
\label{tab:tasks_performance_100}
\begin{adjustbox}{max width=0.9\textwidth}
\small
\renewcommand{\arraystretch}{1.4}
\begin{tabular}{lcccc}
\toprule
\textbf{Task} & \textbf{PAWS (MLP)} & \textbf{PAWS (Transformer)} & \textbf{State (MLP)} & \textbf{State (Transformer)} \\
\midrule
Button Press & \textbf{82$\pm$6} & 80$\pm$5 & 72$\pm$6 & 74$\pm$5 \\
Door Open & 65$\pm$14 & \textbf{70$\pm$8} & 40$\pm$12 & 42$\pm$9 \\
Drawer Open & 39$\pm$12 & \textbf{44$\pm$11} & 22$\pm$10 & 31$\pm$13 \\
Faucet Close & \textbf{68$\pm$8} & 67$\pm$10 & 64$\pm$9 & 60$\pm$8 \\
Lever Pull & 30$\pm$12 & 28$\pm$6 & 16$\pm$11 & \textbf{32$\pm$9} \\
Peg Insert Side & 23$\pm$9 & \textbf{24$\pm$7} & 19$\pm$9 & 20$\pm$9 \\
Plate Slide & \textbf{49$\pm$9} & 48$\pm$6 & 37$\pm$6 & 45$\pm$8 \\
Push Back & 24$\pm$4 & \textbf{26$\pm$5} & 17$\pm$3 & 19$\pm$5 \\
Sweep Into & \textbf{36$\pm$9} & 34$\pm$8 & 33$\pm$6 & 34$\pm$5 \\
Window Close & 91$\pm$7 & \textbf{94$\pm$5} & 79$\pm$8 & 80$\pm$8 \\
\bottomrule
\end{tabular}
\end{adjustbox}
\end{table}

\begin{table}[h]
\centering
\small
\setlength\tabcolsep{6pt}
\caption{Performance results for individual tasks with \textbf{500} preferences. Values represent success rates (\%) $\pm$ 2SE. Best results are highlighted with \textbf{bold}.}
\label{tab:tasks_performance_1000}
\begin{adjustbox}{max width=0.9\textwidth}
\small
\renewcommand{\arraystretch}{1.4}
\begin{tabular}{lcccc}
\toprule
\textbf{Task} & \textbf{PAWS (MLP)} & \textbf{PAWS (Transformer)} & \textbf{State (MLP)} & \textbf{State (Transformer)} \\
\midrule
Button Press & 82$\pm$4 & \textbf{84$\pm$4} & 75$\pm$4 & 78$\pm$5 \\
Door Open & \textbf{98$\pm$1} & 96$\pm$1 & 82$\pm$4 & 64$\pm$8 \\
Drawer Open & \textbf{75$\pm$3} & 74$\pm$3 & 59$\pm$4 & 66$\pm$3 \\
Faucet Close & \textbf{87$\pm$3} & \textbf{87$\pm$3} & 71$\pm$10 & 73$\pm$5 \\
Lever Pull & 55$\pm$4 & \textbf{58$\pm$5} & 56$\pm$4 & 55$\pm$3 \\
Peg Insert Side & \textbf{82$\pm$3} & 81$\pm$4 & 68$\pm$5 & 55$\pm$4 \\
Plate Slide & \textbf{78$\pm$5} & 74$\pm$5 & 54$\pm$6 & 52$\pm$7 \\
Push Back & 53$\pm$3 & \textbf{56$\pm$3} & 44$\pm$4 & 36$\pm$6 \\
Sweep Into & \textbf{74$\pm$4}& \textbf{74$\pm$3} & 67$\pm$3 & 63$\pm$2 \\
Window Close & \textbf{99$\pm$0} & 98$\pm$1 & 91$\pm$2 & 91$\pm$3 \\
\bottomrule
\end{tabular}
\end{adjustbox}
\end{table}
\section{DPPO}
\label{appdx:DPPO}

Results for DPPO~\citep{an2023direct}, which were omitted due to space constraints and comparatively weak performance, are given below.
Meta-World tasks are presented in~\cref{tab:DPPO_mw} and locomotion tasks in~\cref{tab:DPPO_mj}.

\begin{table}[t]
\centering
\caption{DPPO task success (\%) $\pm$ 2SE.}
\label{tab:DPPO_mw}
\begin{tabular}{lcc}
\toprule
\textbf{Task} & \textit{50} & \textit{500} \\
\midrule
Button Press & 15 $\pm$ 4 & 15 $\pm$ 3 \\
Door Open & 15 $\pm$ 5 & 15 $\pm$ 6 \\
Drawer Open & 9 $\pm$ 3 & 13 $\pm$ 2 \\
Faucet Close & 22 $\pm$ 5 & 32 $\pm$ 5 \\
Lever Pull & 6 $\pm$ 1 & 4 $\pm$ 1 \\
Peg Insert Side & 1 $\pm$ 0 & 1 $\pm$ 0 \\
Plate Slide & 17 $\pm$ 5 & 10 $\pm$ 2 \\
Push Back & 4 $\pm$ 1 & 3 $\pm$ 1 \\
Sweep Into & 9 $\pm$ 2 & 9 $\pm$ 2 \\
Window Close & 29 $\pm$ 7 & 35 $\pm$ 3 \\
\bottomrule
\end{tabular}
\end{table}

\begin{table}[t]
\centering
\caption{DPPO average episode returns $\pm$ 2SE.}
\label{tab:DPPO_mj}
\begin{tabular}{lcc}
\toprule
\textbf{Task} & \textit{50} & \textit{500} \\
\midrule
HalfCheetah & -90$\pm$10 & -97$\pm$14 \\
Hopper & 40$\pm$59 & 74$\pm$31 \\
Walker2d & 35$\pm$27 & 32$\pm$13 \\
Ant & -554$\pm$28 & -424$\pm$22 \\
\bottomrule
\end{tabular}
\end{table}
\newpage
\section{Preference Dataset Generation}
\label{appdx:PreferenceDataset}
We evaluate on Meta-World environments with the modifications of CPL~\citep{hejna2024contrastive}.
In contrast to the original Meta-World tasks \citep{yu2020meta}, these are modified by randomizing the goal but including the target in the state observation, as well as randomizing the initial robot position and removing the proprioceptive state history from the observation. Additionally, we evaluate our method on four locomotion tasks, namely Ant, HalfCheetah, Hopper, and Walker2d.
For all tasks, we train one policy with SAC~\citep{haarnoja2018soft} and choose rollouts from four checkpoints throughout training.
For the Meta-World tasks, exact step counts, average return, and success rate are given in \cref{tab:rollout_metaworld}. For the locomotion tasks, the relevant values are given in \cref{tab:rollout_mujoco}.

\begin{table}[h]
\centering
\footnotesize
\setlength\tabcolsep{2pt}
\caption{Dataset quality for the different Meta-World tasks.}
\label{tab:rollout_metaworld}
\begin{adjustbox}{max width=\textwidth}
\small
\renewcommand{\arraystretch}{1.4}
\begin{tabular}{lrccrccrccrcc}
\toprule
\multicolumn{1}{c}{}
& \multicolumn{3}{c}{\textbf{Policy 1}}
& \multicolumn{3}{c}{\textbf{Policy 2}}
& \multicolumn{3}{c}{\textbf{Policy 3}}
& \multicolumn{3}{c}{\textbf{Best Policy}}\\

\cmidrule(lr){2-4}\cmidrule(lr){5-7}\cmidrule(lr){8-10}\cmidrule(lr){11-13}
\textbf{Task} &
\textbf{Step} & \textbf{Return} & \textbf{Success} &
\textbf{Step} & \textbf{Return} & \textbf{Success} &
\textbf{Step} & \textbf{Return} & \textbf{Success} &
\textbf{Step} & \textbf{Return} & \textbf{Success} \\
\midrule
Button Press & 40k  & 509  & 16\% & 70k  & 1335 & 73\% & 150k & 1320 & 74\% &240k     & 1709 & 96\% \\
Door Open    & 30k  & 462  & 2\%    & 50k  & 708  & 41\%   & 70k  & 1314 & 83\%   &  830k   & 1994 & 100\%  \\
Drawer Open  & 210k & 1279 & 8\% & 270k & 1535 & 61\% & 290k & 1640 & 76\% & 350k & 1786 & 91\% \\
Faucet Close & 30k  & 849  & 1\% & 60k  & 1346 & 36\% & 90k  & 1451 & 51\% & 140k & 2124 & 98\% \\
Lever Pull   & 30k  & 204  & 0\% & 190k & 265  & 23\% & 300k & 382  & 52\% & 640k & 720  & 80\% \\
Peg Insert Side   & 340k & 735  & 4\% & 390k & 1201 & 33\% & 410k & 1507 & 86\% & 480k & 1760 & 97\% \\
Plate Slide  & 50k  & 487  & 9\% & 60k  & 448  & 8\% & 120k & 1543 & 78\% & 250k & 2006 & 99\% \\
Push Back    & 280k & 47   & 12\% & 290k & 102  & 20\% & 410k & 573  & 57\% & 600k & 1622 & 97\% \\
Sweep Into   & 50k  & 142  & 0\% & 150k & 490  & 52\%   & 300k & 1176 & 81\%   & 910k     & 1958 & 97\%   \\
Window Close & 30k  & 240  & 5\% & 70k  & 669  & 48\% & 80k  & 906  & 14\% & 120k & 1524 & 99\% \\
\bottomrule
\end{tabular}
\end{adjustbox}
\end{table}

\begin{table}[h]
\centering
\footnotesize
\setlength\tabcolsep{4pt}
\caption{Dataset quality for the different locomotion tasks.}
\label{tab:rollout_mujoco}
\begin{adjustbox}{max width=\textwidth}
\small
\renewcommand{\arraystretch}{1.4}
\begin{tabular}{lrcrcrcrcc}
\toprule
\multicolumn{1}{c}{}
& \multicolumn{2}{c}{\textbf{Policy 1}}
& \multicolumn{2}{c}{\textbf{Policy 2}}
& \multicolumn{2}{c}{\textbf{Policy 3}}
& \multicolumn{2}{c}{\textbf{Policy 4}}\\
\cmidrule(lr){2-3}\cmidrule(lr){4-5}\cmidrule(lr){6-7}\cmidrule(lr){8-9}
\textbf{Task} &
\textbf{Step} & \textbf{Return} &
\textbf{Step} & \textbf{Return} &
\textbf{Step} & \textbf{Return} &
\textbf{Step} & \textbf{Return} \\
\midrule
Ant         & 10k  & 245  & 500k & 376  & 1M   & 660  & 2M   & 935  \\
HalfCheetah & 30k  & 76   & 90k  & 1071 & 200k & 1583 & 300k & 2499 \\
Hopper      & 30k  & 339  & 160k & 289  & 260k & 693  & 400k & 788  \\
Walker2d    & 60k  & 59   & 280k & 260  & 340k & 527  & 500k & 1093 \\
\bottomrule
\end{tabular}
\end{adjustbox}
\end{table}

\begin{figure*}[h]
  \centering
  \begin{subfigure}[t]{0.19\textwidth}
    \includegraphics[width=\linewidth]{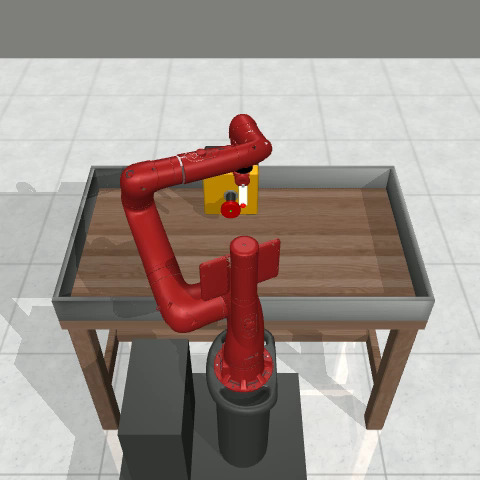}
    \caption{Button Press}\label{fig:mw:a}
  \end{subfigure}\hfill
  \begin{subfigure}[t]{0.19\textwidth}
    \includegraphics[width=\linewidth]{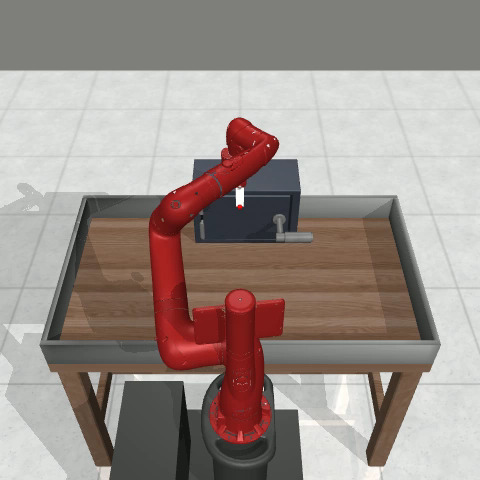}
    \caption{Door Open}\label{fig:mw:b}
  \end{subfigure}\hfill
  \begin{subfigure}[t]{0.19\textwidth}
    \includegraphics[width=\linewidth]{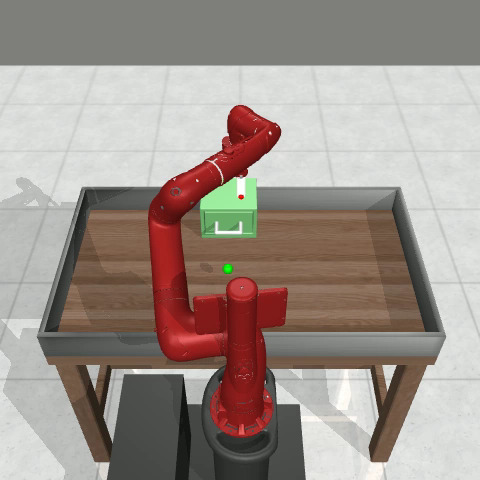}
    \caption{Drawer Open}\label{fig:mw:c}
  \end{subfigure}\hfill
  \begin{subfigure}[t]{0.19\textwidth}
    \includegraphics[width=\linewidth]{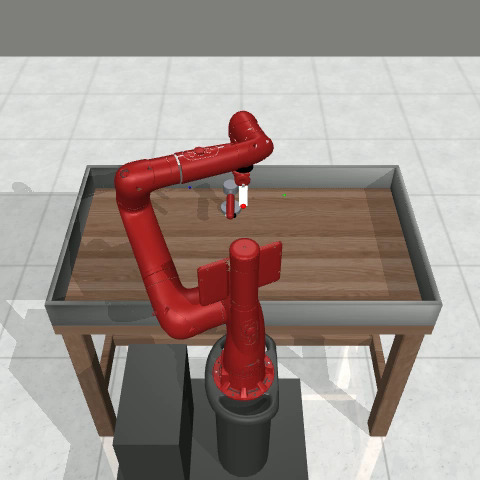}
    \caption{Faucet Close}\label{fig:mw:d}
  \end{subfigure}\hfill
  \begin{subfigure}[t]{0.19\textwidth}
    \includegraphics[width=\linewidth]{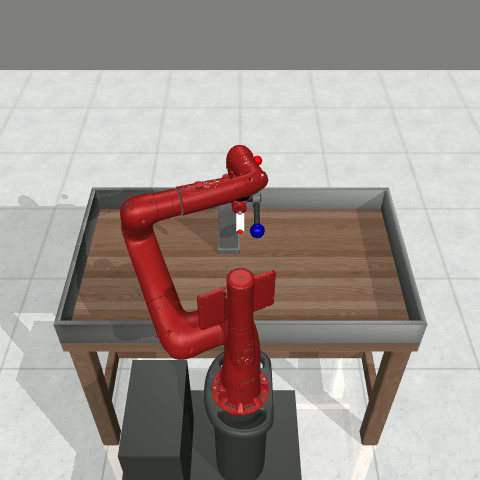}
    \caption{Lever Pull}\label{fig:mw:e}
  \end{subfigure}

  \vspace{4pt}

  \begin{subfigure}[t]{0.19\textwidth}
    \includegraphics[width=\linewidth]{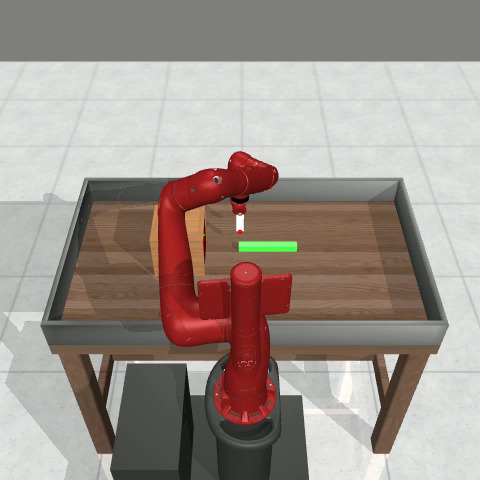}
    \caption{Peg Insert Side}\label{fig:mw:f}
  \end{subfigure}\hfill
  \begin{subfigure}[t]{0.19\textwidth}
    \includegraphics[width=\linewidth]{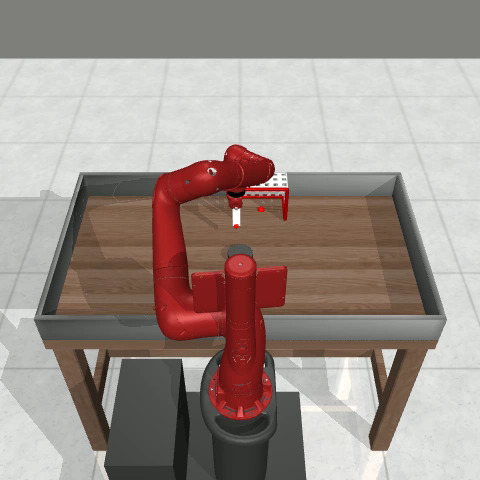}
    \caption{Plate Slide}\label{fig:mw:g}
  \end{subfigure}\hfill
  \begin{subfigure}[t]{0.19\textwidth}
    \includegraphics[width=\linewidth]{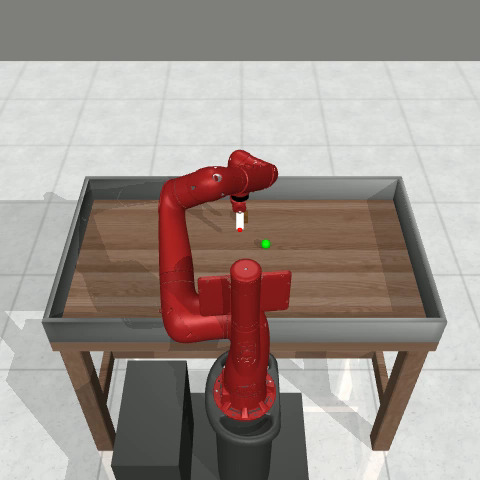}
    \caption{Push Back}\label{fig:mw:h}
  \end{subfigure}\hfill
  \begin{subfigure}[t]{0.19\textwidth}
    \includegraphics[width=\linewidth]{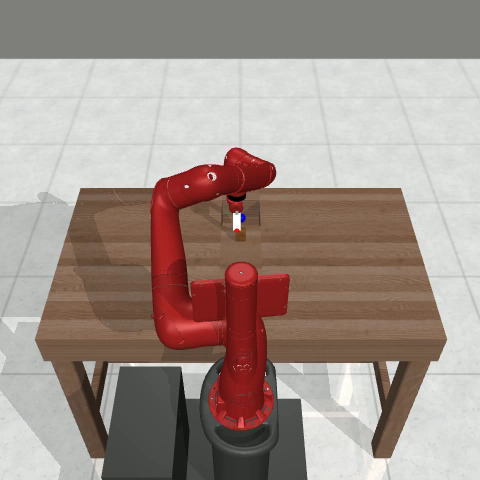}
    \caption{Sweep Into}\label{fig:mw:i}
  \end{subfigure}\hfill
  \begin{subfigure}[t]{0.19\textwidth}
    \includegraphics[width=\linewidth]{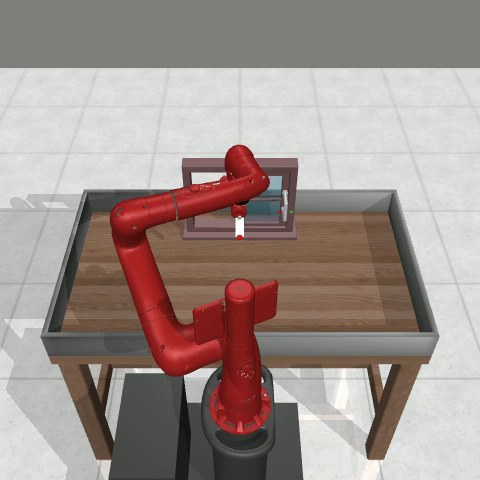}
    \caption{Window Close}\label{fig:mw:j}
  \end{subfigure}

  \caption{Meta-World manipulation tasks used in our experiments. Each panel shows an initial configuration (a–j).}
  \label{fig:metaworld}
\end{figure*}

\subsection{Meta-World Task Descriptions}
\label{appdx:Mw_desc}
We evaluate on 10 Meta-World manipulation tasks shown in \cref{fig:metaworld}: \emph{(a) Button Press} press a small button on the panel until it activates; \emph{(b) Door Open} open a hinged door by pulling the handle to a target angle; \emph{(c) Drawer Open} pull the drawer along its rail until it reaches the goal extension; \emph{(d) Faucet Close} rotate the faucet handle clockwise until it is fully off; \emph{(e) Lever Pull} pull the short lever down through a quarter turn; \emph{(f) Peg Insert Side} insert a cylindrical peg into a horizontal side hole without jamming; \emph{(g) Plate Slide} push the plate across the table into the marked goal region; \emph{(h) Push Back} push the movable puck backward to the target position; \emph{(i) Sweep Into} sweep a small object across the surface into a container opening; \emph{(j) Window Close} push the sliding window pane along its track until fully closed.

\subsection{Locomotion Task Descriptions}
We evaluate four different locomotion tasks~\citep{towers2024gymnasium}, namely Ant, HalfCheetah, Hopper, and Walker2d. All tasks have the same episode length of 250 steps, and episodes terminate only when the final step is reached. There are no early-termination conditions, which are otherwise present in \citet{towers2024gymnasium}, because this is required for consistent generation of the trajectories used for preferences. The tasks are visualized in \cref{fig:locomotion}.

\begin{figure*}[h]
  \centering
   \begin{subfigure}[t]{0.22\textwidth}
    \includegraphics[width=\linewidth]{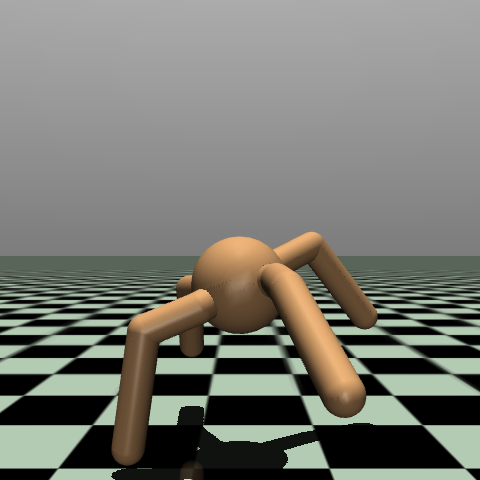}
    \caption{Ant}\label{fig:lm:a}
      \end{subfigure}\hfill
  \begin{subfigure}[t]{0.22\textwidth}
    \includegraphics[width=\linewidth]{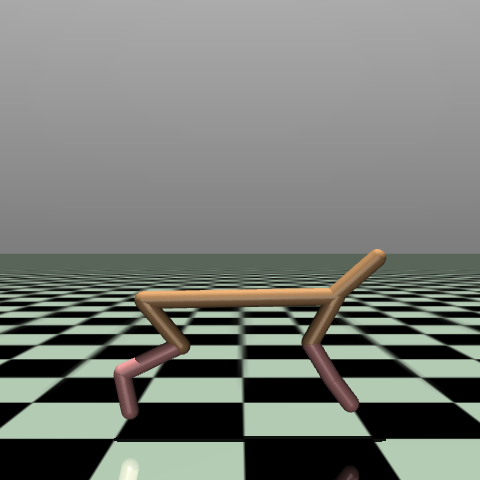}
    \caption{HalfCheetah}\label{fig:lm:b}
  \end{subfigure}\hfill
  \begin{subfigure}[t]{0.22\textwidth}
    \includegraphics[width=\linewidth]{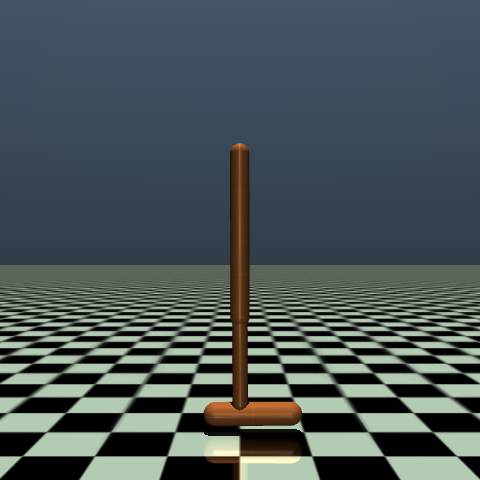}
    \caption{Hopper}\label{fig:lm:c}
  \end{subfigure}\hfill
  \begin{subfigure}[t]{0.22\textwidth}
    \includegraphics[width=\linewidth]{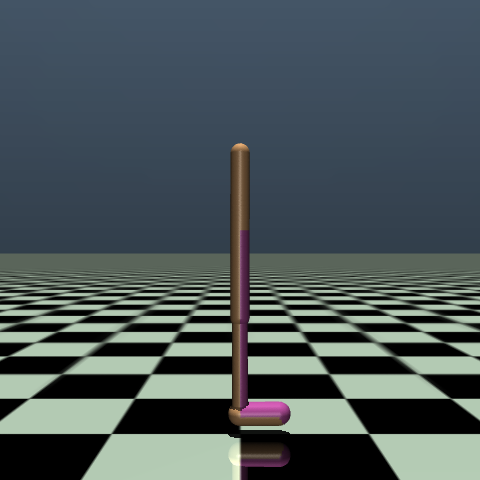}
    \caption{Walker2d}\label{fig:lm:d}
  \end{subfigure}\hfill
  \caption{Locomotion tasks.}
  \label{fig:locomotion}
\end{figure*}
\section{Algorithm Details and Hyperparameters}
\label{app:hp}
The method hyperparameters are listed in the following tables: \cref{tab:hyperparameters} for the baselines and \cref{tab:hyperparameters_paws} for our method. Our method is also summarized in \cref{alg:offline-paws}.

\begin{table}[h]
\centering
\small
\setlength\tabcolsep{6pt}
\caption{Hyperparameters for the baselines.}
\label{tab:hyperparameters}
\begin{adjustbox}{max width=0.7\textwidth}
\small
\renewcommand{\arraystretch}{1.4}
\begin{tabular}{lccccc}
\toprule
\textbf{Hyperparameter} & \textbf{CPL} & \textbf{CPL (KL)} & \textbf{P-IQL} & \textbf{Pref Trans.} & \textbf{IPL} \\
\midrule
Learning Rate & 0.0001 & 0.0001 & 0.0003 & 0.0003 & 0.0001 \\
Temp $\alpha$ & 0.1 & 0.2 & - & - & - \\
Bias $\lambda$ & 0.5 & 0.75 & - & - & - \\
$\gamma$ & - & - & 0.99 & 0.99 & 0.99 \\
Expectile $\tau$ & - & - & 0.7 & 0.7 & 0.75 \\
Temperature $\beta$ & - & - & 0.3333 & 3.0 & 0.3333 \\
Hidden Dims & - & - & - & (256, 256) & (512, 512) \\
Soft Target $\tau$ & - & - & - & 0.005 & 0.005 \\
Reward Net Steps & - & - & 50000 & - & - \\
Evaluation Step  & 5000 & 5000 & 5000 & 5000 & 5000 \\
\bottomrule
\end{tabular}
\end{adjustbox}
\end{table}

\begin{table}[h]
\centering
\small
\setlength\tabcolsep{6pt}
\caption{Hyperparameters for PAWS.}
\label{tab:hyperparameters_paws}
\begin{adjustbox}{max width=0.85\textwidth}
\small
\renewcommand{\arraystretch}{1.4}
\begin{tabular}{lcc}
\toprule
\textbf{Hyperparameter} & \textbf{PAWS (MLP)} & \textbf{PAWS (Transformer)} \\
\midrule
\textbf{Actor networks} & & \\

Learning rate & 0.0003 & 0.0003 \\
Dropout & 0.25 & 0.25 \\
Hidden dimension & 512 & 512 \\
Hidden depth & 2 & 2 \\
Evaluation Step  & 250 & 250 \\
Effective sample size $n^*_{\mathrm{eff}}$ (Meta-World / Locomotion) & $10\%$ / $30\%$ & $10\%$ / $30\%$ \\
\midrule
\textbf{Reward networks} & & \\

Learning rate & 0.0003 & 0.0003 \\
Max Len & - & 64 \\
Dropout & - & 0.1 \\
Hidden dimension & 512 & 512 \\
Number of heads & - & 8 \\
Number of layers & 3 & 4 \\
Position encoding & - & learned \\
Min. early stopping value $\alpha$ & 0.995 & 0.995\\
\bottomrule
\end{tabular}
\end{adjustbox}
\end{table}

\begin{algorithm}[t]
\caption{\textbf{PAWS: Preference Learning with Segment-Level Advantage Optimization}}
\label{alg:offline-paws}
\begin{algorithmic}[1]

\State \textbf{Input:} Offline segment dataset 
$\mathcal{D}=\{\tau_i\}_{i=1}^K$ sampled from $p_{\mathcal{D}}(\tau)$
\State \textbf{Input:} Offline preference dataset 
$\mathcal{D}_{\mathrm{pref}} = \{(\tau_i^+, \tau_i^-)\}_{i=1}^n$, 
where $\tau_i^+ \succ \tau_i^-$
\State \textbf{Input:} Advantage model $A_\phi(s,a)$, policy $\pi_\theta(a\mid s)$
\State Initialize $\phi$ (advantage) and $\theta$ (policy)

\Statex \vspace{2mm}\textbf{1. Advantage learning on preference data}
\While{advantage parameters $\phi$ are not converged}
 
    \State Minimize preference loss (\cref{eq::AdvOpt})
    \[
        \mathcal{L}_{\mathrm{pref}}(\phi)
        = -\frac{1}{n}\sum_{(\tau^+,\tau^-)\in\mathcal{D}_{\mathrm{pref}}}
        \log\sigma(A_\phi(\tau^+)-A_\phi(\tau^-))
    \]
    \State Update $\phi \leftarrow \phi - \alpha_\phi \nabla_\phi \mathcal{L}_{\mathrm{pref}}$
\EndWhile
\Statex \vspace{2mm}\textbf{2. Lagrange Multiplier Computation}
\State Find $\lambda$ resulting in desired $n_{\mathrm{eff}}^*$ using \cref{eq::n_eff}
\State $\lambda \gets \operatorname{argmin}_\lambda \left| n_{\mathrm{eff}}^* - n_{\mathrm{eff}}(\lambda) \right|$

\Statex \vspace{2mm}\textbf{3. Policy update}
\While{policy parameters $\theta$ are not converged}
    \For{each $\tau_i \in \mathcal{D}$}
        \State \textbf{Compute importance weights (\cref{eq::ImpWeightDef})}
        \State $w_i \gets \exp\!\big(\tfrac{1}{\lambda} A_\phi(\tau_i)\big)$
        \State {Self-normalize for numerical stability}
        \State $w_i \gets \frac{w_i}{\sum_j w_j}$
    \EndFor

    \State \textbf{Policy update using weighted maximum-likelihood (\cref{eq::grad_ml_loss_sa})}
    \State Compute gradient:
    \[
        \nabla_\theta\mathcal{L}
        = \frac{1}{K} \sum_i w_i \sum_t \nabla_\theta \log\pi_\theta(a^i_t \mid s^i_t)
    \]
    \State Update $\theta \leftarrow \theta + \alpha_\theta \nabla_\theta \mathcal{L}$

\EndWhile

\State \textbf{return} $\pi_\theta$

\end{algorithmic}
\end{algorithm}

\section{Ablation: Spearman's Rank Correlation Coefficients}
\label{sec:spearman}
To compute the correlation of the learned model's action likelihoods with the ground-truth advantage, we compute Spearman's correlation coefficient between the log likelihoods of the expert policy of a segment and the log likelihoods coming from our trained policy for the same segment.
The Spearman's correlation coefficient of sequences $X_i$ and $Y_i$ first gives each value an integer rank in increasing order and in particular for unique values, it holds that 
\begin{equation}
    \operatorname{rank}(X) = |\{X_j \le X\}_j|.
\end{equation}
Then the sequences $\operatorname{rank}(X_i)$ and $\operatorname{rank}(Y_i)$ are compared using the Pearson correlation coefficient.
In contrast to the direct Pearson correlation coefficient, this also measures nonlinear but monotonic correlation between the two sequences.
This is relevant for measuring the credit assignment as this removes the impact of the absolute quality of the policy likelihood fit to the expert policy likelihood, that is present in a linear correlation measurement, but measures whether the same state-action pairs of the sequence have high or low likelihood, hence inferred advantage, \emph{within} the segment. 

We use the Transformer-based advantage function and evaluate after 10{,}000 policy update steps.
For each of $K=1{,}000$ segments per task, we determine the rank correlation between the predicted policy likelihoods and the ground-truth expert policy likelihoods, that were used for preference label generation.
Practically, we rank the log likelihoods which is equivalent due to the monotonicity of the logarithm.
The reported values in~\cref{tab:spearman_correlation} are averages over the training results of 5 seeds.

\begin{table}[H]
\centering
\small
\setlength\tabcolsep{4pt}
\caption{Spearman correlation coefficient between the action likelihoods of the learned policy and the expert policy for each Meta-World task. We compare the \textbf{State} versus the \textbf{Segment} representation.}
\label{tab:spearman_correlation}
\begin{adjustbox}{max width=\textwidth}
\footnotesize
\renewcommand{\arraystretch}{1.4}
\begin{tabular}{lcccccccccc|c}
\toprule
& \textbf{Button Press} & \textbf{Door Open} & \textbf{Drawer Open} & \textbf{Faucet Close} & \textbf{Lever Pull} & \textbf{Peg Insert Side} & \textbf{Plate Slide} & \textbf{Push Back} & \textbf{Sweep Into} & \textbf{Window Close} & \textbf{Mean} \\
\midrule
\textbf{State}   & 0.06           &          0.11 &         -0.34 &          0.35 & \textbf{-0.1} &          0.03 &          0.26 &          0.05 &         -0.04 &          0.16 & \textbf{0.054} \\
\textbf{Segment} & \textbf{0.36 } & \textbf{0.21} & \textbf{0.03} & \textbf{0.57} &         -0.26 & \textbf{0.24} & \textbf{0.27} & \textbf{0.20} & \textbf{0.00} & \textbf{0.55} & \textbf{0.217} \\
\bottomrule
\end{tabular}
\end{adjustbox}
\vspace{-2.5mm}
\end{table}
\section{Human Preferences}
\label{app:human}
We collected preferences from 10 non-author human labelers on two Meta-World tasks, Button Press and Door Open. Each labeler provided 50 pairwise comparisons per task from 100 segments, yielding 500 comparisons per task in total. The preference pairs were generated in the same way as in \cref{appdx:PreferenceDataset}, with the labels provided by humans rather than the oracle. To collect the labels, we used the GUI shown in \cref{fig:GUI}, which presented the two trajectory videos side by side and required the labeler to choose one of them as preferred (no tie option was provided). Each labeler was provided with the task description as defined in \cref{appdx:Mw_desc}.

The data collection does not constitute human subjects research in the regulated sense. The participants viewed pairs of robot execution videos and indicated preferences, with no personal data collected, anonymous participation, and no intervention or risk. The object of study is the robot executions; participants served purely as annotators, analogous to crowdsourced labeling. Our institution does not mandate ethics approval for anonymous, minimal-risk annotation of this kind.

To assess statistical significance, we performed pairwise Welch's t-tests comparing each PAWS variant against all baselines. On Door Open, PAWS (MLP) significantly outperforms all baselines, and PAWS (Trans.) significantly outperforms all methods except Preference Transformer. After applying Bonferroni correction ($\alpha=0.05/24=0.0021$, where $24 = 2~\text{tasks} \times 2~\text{PAWS variants} \times 6~\text{baselines}$), PAWS (MLP) on Door Open remains significant against BC, P-IQL, CPL, CPL+KL, and IPL ($p_\text{Bonf}<0.05$). Under the less conservative Benjamini-Hochberg FDR correction ($q=0.05$), PAWS (MLP) on Door Open additionally retains significance against Pref Trans. ($p_\text{BH}<0.05$), and both PAWS variants retain significance against BC on Button Press ($p_\text{BH}<0.01$).

\begin{figure}
    \centering
    \includegraphics[width=1\linewidth]{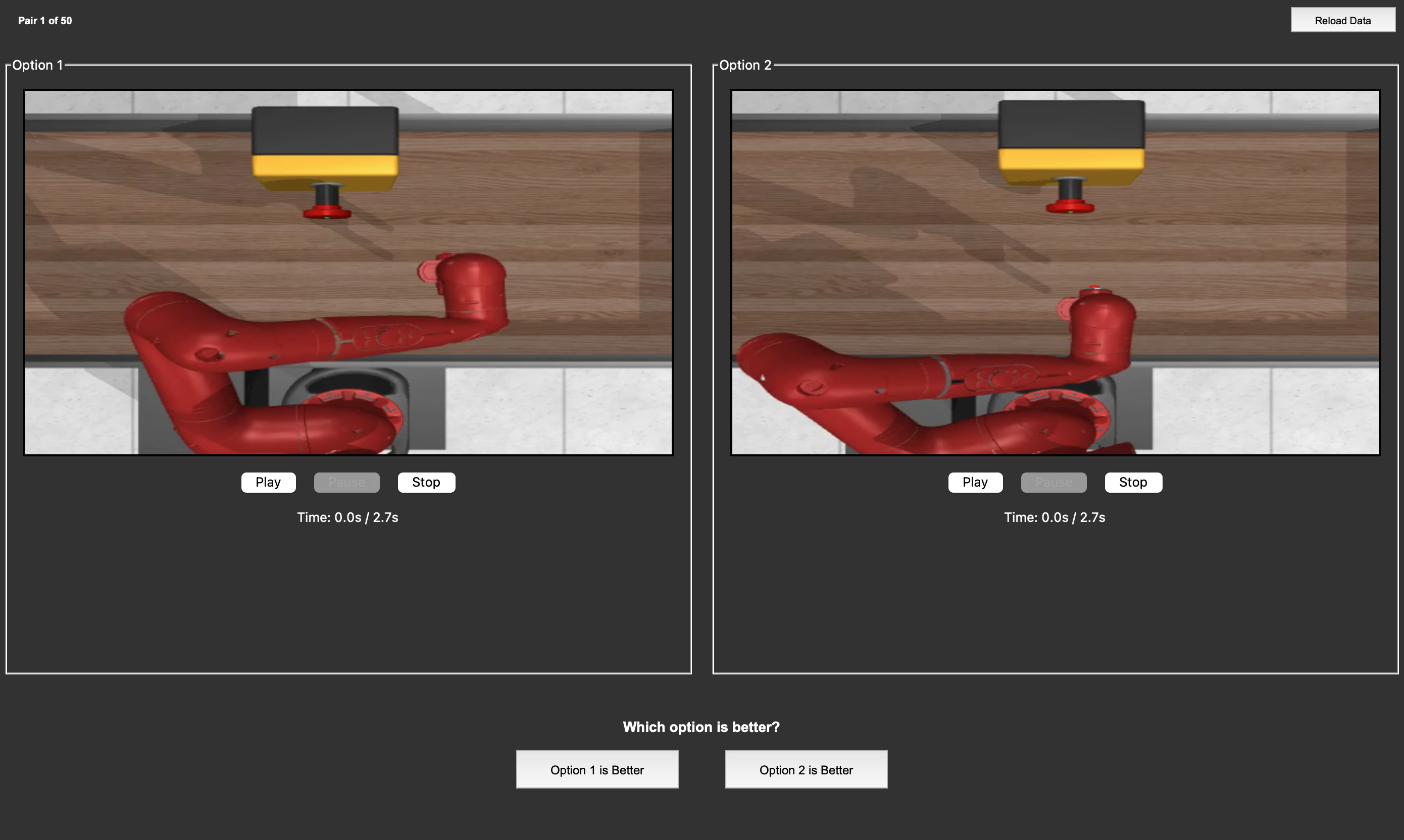}
    \caption{GUI for collecting preference labels from humans, shown here for the Button Press task. The labeler sees two robot execution videos side by side, with Play, Pause, and Stop controls for each, and selects which option is better. The header shows progress through the 50 pairs for the current task. The Door Open task uses the same layout with different videos.}
    \label{fig:GUI}
\end{figure}
%%%%%%%%%%%%%%%%%%%%%%%%%%%%%%%%%%%%%%%%%%%%%%%%%%%%%%%%%%%%%%%%%%%%%%%%%%%%%%%
%%%%%%%%%%%%%%%%%%%%%%%%%%%%%%%%%%%%%%%%%%%%%%%%%%%%%%%%%%%%%%%%%%%%%%%%%%%%%%%

\end{document}